%% file: main.tex
\documentclass[lettersize,journal]{IEEEtran}
\usepackage{amsmath,amsfonts}
\usepackage{algorithmic}
\usepackage{array}
\usepackage{textcomp}
\usepackage{stfloats}
\usepackage{url}
\usepackage{verbatim}
\usepackage{graphicx}
\def\BibTeX{{\rm B\kern-.05em{\sc i\kern-.025em b}\kern-.08em
    T\kern-.1667em\lower.7ex\hbox{E}\kern-.125emX}}
\usepackage{balance}

\input{./deps.tex}

\begin{document}
\title{Parameterized Convex Minorant for Objective Function Approximation in Amortized Optimization}

\author{
    Jinrae Kim
    and
    Youdan Kim, \textit{Senior Member, IEEE}
    \thanks{Jinrae Kim is with the Department of Aerospace Engineering,
    Seoul National University, Seoul 08826, Republic of Korea (e-mail: kjl950403@snu.ac.kr)}
    \thanks{Youdan Kim is with the Department of Aerospace Engineering,
        Institute of Advanced Aerospace Technology,
    Seoul National University, Seoul 08826, Republic of Korea (e-mail: ydkim@snu.ac.kr)}
}
\markboth{Journal of \LaTeX\ Class Files,~Vol.~18, No.~9, September~2020}%
{How to Use the IEEEtran \LaTeX \ Templates}

\maketitle

\input{./sections/0_abstract.tex}


\input{./sections/1_introduction.tex}
\input{./sections/2_preliminaries.tex}
\input{./sections/3_main_results.tex}
\input{./sections/4_numerical_simulation.tex}

\input{./sections/5_conclusion.tex}

\section*{Acknowledgments}
This work was supported by the National Research Foundation of Korea (NRF) grant funded by the Korean government (MSIT) (No. 2019R1A2C2083946).

\begin{appendices}
  \input{./appendices/appendix_A.tex}
  \input{./appendices/appendix_B.tex}
\end{appendices}

\bibliography{ref.bib}
\bibliographystyle{IEEEtran}

\begin{IEEEbiographynophoto}{Jinrae Kim}
is a Ph.D. candidate in the Department of Aerospace Engineering at Seoul National University. He received the B.S. degree in mechanical and aerospace engineering from Seoul National University, Republic of Korea, in 2017. His current research interests include machine learning, optimization, and control for robotics and aerospace engineering applications.
\end{IEEEbiographynophoto}

\begin{IEEEbiographynophoto}{Youdan Kim}
received B.S. and M.S. degrees in aeronautical engineering from Seoul National University, Republic of Korea, in 1983 and 1985, respectively, and the Ph.D. degree in aerospace engineering from Texas A\&M University in 1990. He joined the faculty of Seoul National University in 1992, where he is currently a Professor with the Department of Aerospace Engineering. His current research interests include aircraft control system design, reconfigurable control system design, path planning, and guidance techniques for aerospace systems.
\end{IEEEbiographynophoto}

\end{document}

%% file: deps.tex
\usepackage{amsthm,amssymb}
\usepackage{hyperref}
\DeclareMathOperator*{\argmin}{arg\,min}
\usepackage{subcaption}
\usepackage{tablefootnote}
\usepackage[dvipsnames]{xcolor}
\newtheorem{theorem}{Theorem} 
\newtheorem{lemma}[theorem]{Lemma} 
\newtheorem{definition}{Definition} 
\newtheorem{assumption}[theorem]{Assumption}

\newcommand\lref[1]{Lemma~\ref{#1}}
\newcommand\aref[1]{Assumption~\ref{#1}}

\usepackage{tikz}

%% file: sections/0_abstract.tex
\begin{abstract}%
  Parameterized convex minorant (PCM) method is proposed
  for the approximation of the objective function in amortized optimization.
  In the proposed method,
  the objective function approximator is expressed by the sum of a PCM and a nonnegative gap function,
  where the objective function approximator is bounded from below by the PCM convex in the optimization variable.
  The proposed objective function approximator is a universal approximator for continuous functions,
  and the global minimizer of the PCM attains the global minimum of the objective function approximator.
  Therefore,
  the global minimizer of the objective function approximator can be obtained by a single convex optimization.
  As a realization of the proposed method,
  extended parameterized log-sum-exp network is proposed 
  by utilizing a parameterized log-sum-exp network as the PCM.
  Numerical simulation is performed 
  for parameterized non-convex objective function approximation and
  for learning-based nonlinear model predictive control
  to demonstrate the performance and characteristics of the proposed method.
  The simulation results support that the proposed method
  can be used to learn objective functions and
  to find a global minimizer reliably and quickly by using convex optimization algorithms.
\end{abstract}

%% file: sections/1_introduction.tex
\begin{figure*}
  \centering
  \begin{subfigure}[b]{0.31\linewidth}
    \centering
    \includegraphics[width=1.0\linewidth]{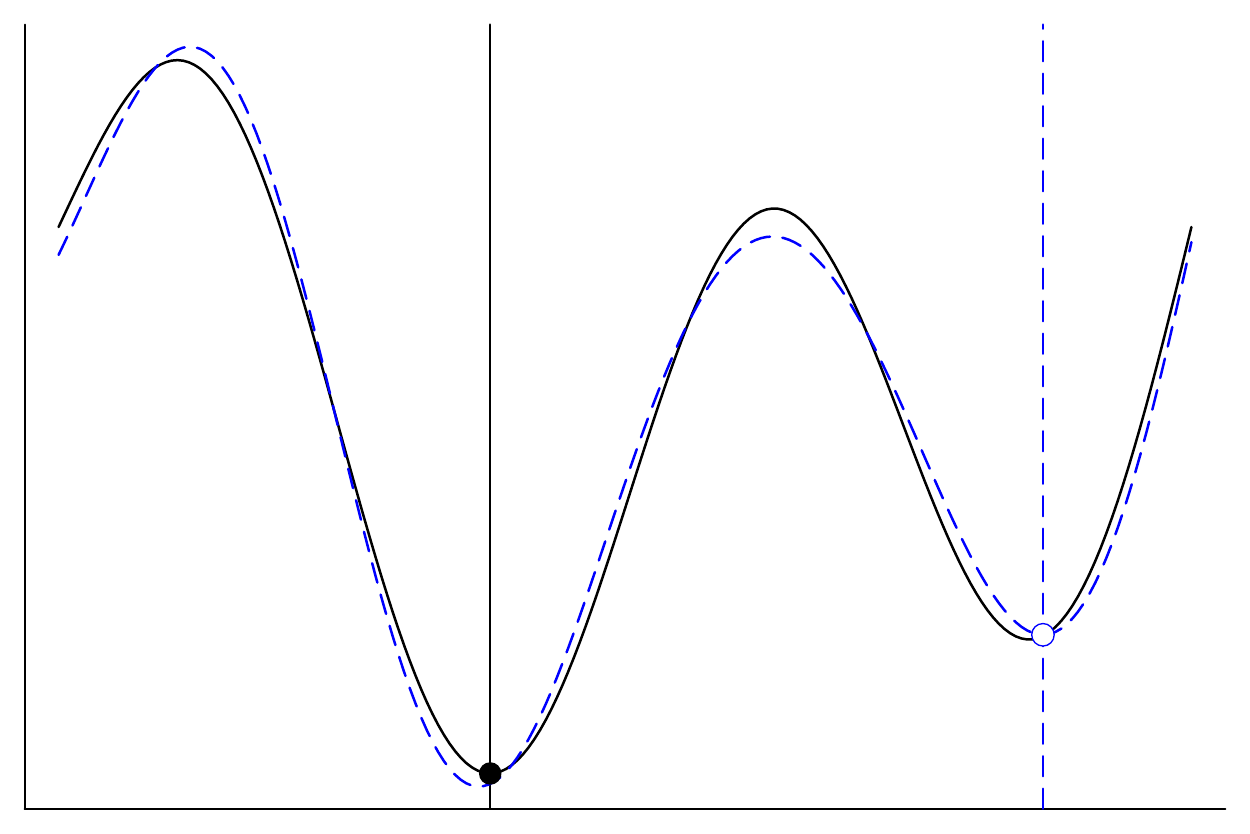}
    \caption{}
  \end{subfigure}
  \begin{subfigure}[b]{0.31\linewidth}
    \centering
    \includegraphics[width=1.0\linewidth]{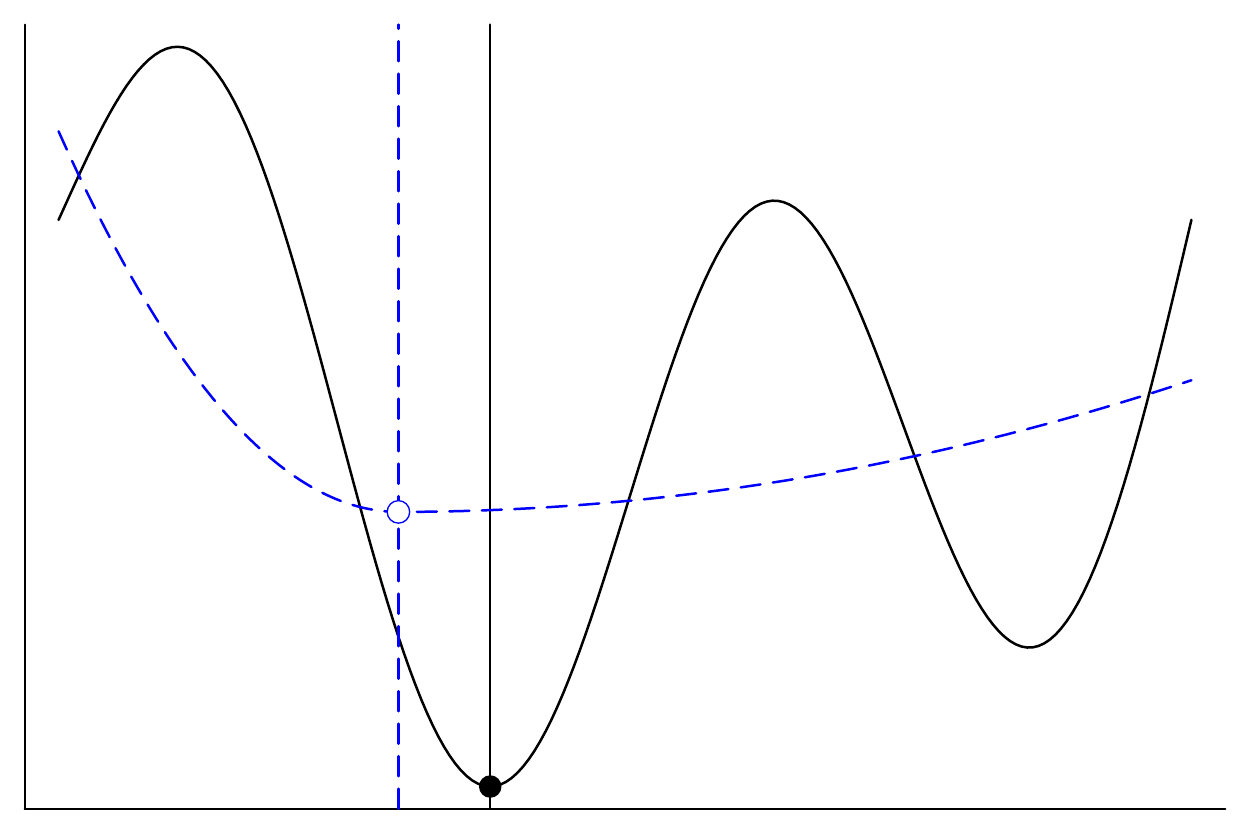}
    \caption{}
  \end{subfigure}
  \begin{subfigure}[b]{0.31\linewidth}
    \centering
    \includegraphics[width=1.0\linewidth]{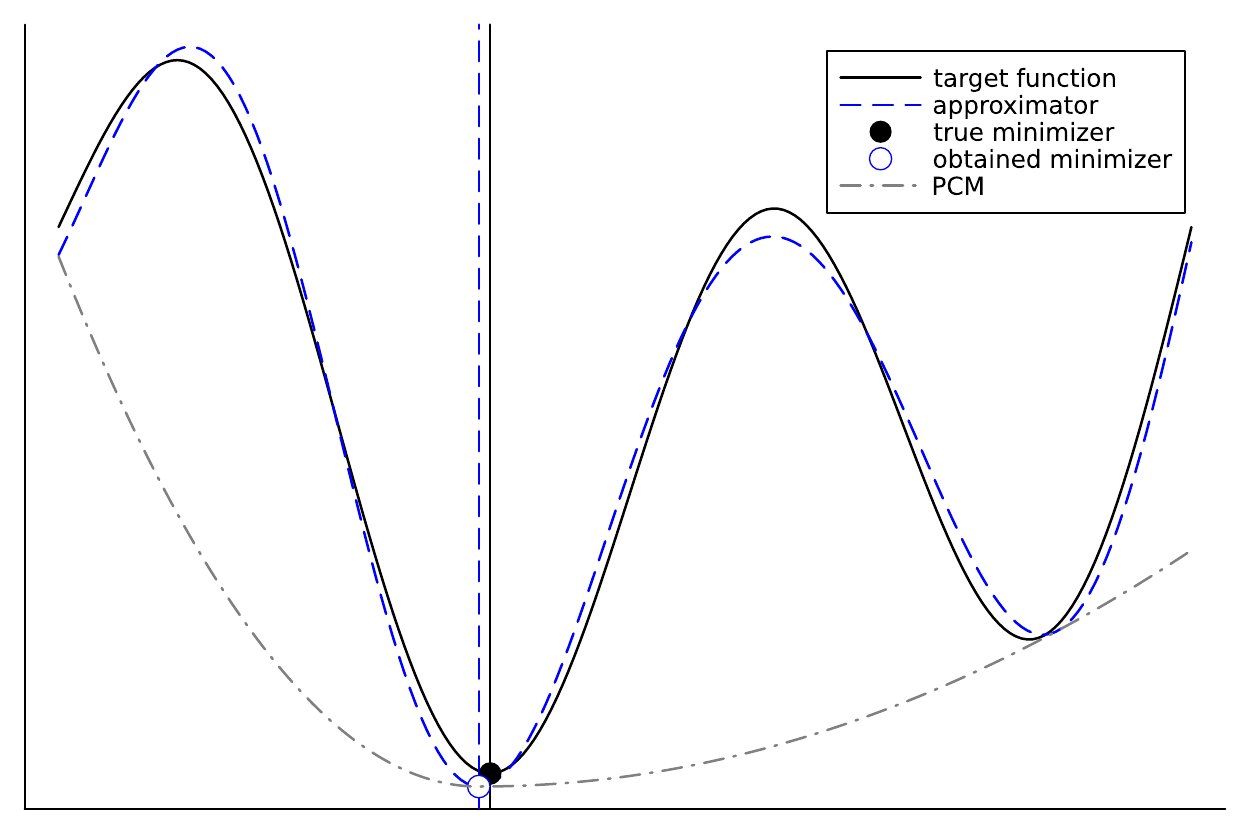}
    \caption{}
  \end{subfigure}
  \caption{Concept visualization.
  (a) Approximators with high expressiveness and local optimality (e.g. FNN, DLSE),
  (b) approximators with global optimality and restrictive expressiveness (e.g. PLSE),
  and (c) the proposed approximator can find a global minimizer by minimizing the parameterized convex minorant (PCM), which shares a global minimizer with the proposed approximator.
}
  \label{fig:concept}
\end{figure*}

\section{Introduction}
Parametric optimization finds a minimizer of the objective function that depends on the parameter.
A wide range of engineering applications can be viewed as parametric optimization,
including reinforcement learning and optimal control.
The involved optimization problem of the parametric optimization is, however, non-convex in general,
and the objective function itself might be unknown.
Therefore,
it is challenging to find a global minimizer quickly and reliably.

In machine learning,
\textit{amortized optimization} was proposed as a framework of learning-based parametric optimization~\cite{amosTutorialAmortizedOptimization2023}.
Amortized optimization is basically to learn certain related functions in advance
to quickly find the minimizer of the parametric optimization in operation.
Amortized optimization methods can be classified by what functions are learned:
minimizer function approximation and objective function approximation methods.
In each method,
an approximator is used to learn the minimizer function or objective function, respectively.
The objective function approximation method is attractive because
the objective can be evaluated for different optimization variables and parameters by using the trained objective function approximator.
However,
it is challenging to find a global minimizer in the objective function approximation method
because it typically involves non-convex optimization to retrieve the approximate minimizer from the highly expressive objective function approximator.
Therefore,
in most studies on amortized optimization,
the minimizer function approximation method
has been utilized under the assumption that there exists
a unique single-valued continuous minimizer function,
which restricts the range of applications~\cite{amosTutorialAmortizedOptimization2023}.

For the objective function approximation method,
one may use approximators with specific \textit{shapes} to avoid the challenge of local optimality.
A promising approach is to use \textit{shape-preserving} approximators exploiting convexity.
That is,
an approximator is (parameterized) convex and is also capable of approximating any (parameterized) convex functions in terms of the universal approximation theorem,
where an objective function is said to be \textit{parameterized convex} if the objective function is convex in the optimization variable for
a given parameter~\cite{kimParameterizedConvexUniversal2022b}.
This approach can avoid local optimality while having good convergence properties (e.g., global convergence and a fast convergence rate)
because a global minimizer of many convex functions can be obtained reliably and quickly by using state-of-the-art convex optimization solvers~\cite{boydConvexOptimization2004,malyutaConvexOptimizationTrajectory2022,liuSurveyConvexOptimization2017}.
For example,
max-affine (MA) and log-sum-exp (LSE) networks were proposed as shape-preserving universal approximators for convex functions,
where the LSE network is a smoothed version of
the MA network~\cite{calafioreLogSumExpNeuralNetworks2020}.
Because the projection of a convex function is also convex,
the LSE network was utilized to approximate the state-action value function (Q-function) in reinforcement learning,
an example of the objective function approximation method in amortized optimization~\cite{calafioreEfficientModelFreeQFactor2020}.
Kim and Kim proposed parameterized MA (PMA) and parameterized LSE (PLSE) networks
as the extension of the MA and LSE networks, respectively,
where PMA and PLSE are shape-preserving universal approximators for parameterized convex functions~\cite{kimParameterizedConvexUniversal2022b}.
The PLSE network was applied for optimal gain prediction of VTOL aircraft~\cite{kimVTOLAircraftOptimal2023}.
Nonetheless, as concerned in the study,
shape-preserving approximators may fail to approximate the target objective function having more general shapes,
which may provide degraded approximate minimizers.
This can be interpreted as the \textit{trade-off between the expressiveness and shape} in objective function approximation.
To mitigate this trade-off issue,
difference of LSE (DLSE) network was proposed~\cite{calafioreUniversalApproximationResult2020}.
The DLSE network is a universal approximator for continuous functions,
which can obtain the local minimum by iteratively solving convex optimization subproblems,
featured as difference of convex algorithms (DCA)~\cite{lethiDCProgrammingDCA2018}.
Although DCA has good convergence properties such as global and linear convergence,
DLSE still suffers from local optimality even with iterations of convex optimization.

In this study,
to overcome the limitations of existing studies and resolve the trade-off between expressiveness and shape,
\textit{parameterized convex minorant} (PCM) method is proposed as a new objective function approximation method in amortized optimization.
In the proposed method,
the objective function approximator consists of a PCM and a nonnegative gap function,
where the objective function approximator is bounded from below by the PCM.
The concept of the proposed method is illustrated in \autoref{fig:concept},
compared to other types of approximators.
This study reveals that, for a given parameter,
the minimizer of the objective function approximator
can be found by minimizing the PCM instead,
where the corresponding minimization problem of the PCM is convex optimization.
Additionally,
it is proven that the objective function approximator is a universal approximator for continuous functions
when shape-preserving universal approximators for parameterized convex continuous and continuous functions are utilized as the PCM and the gap function, respectively.
These results imply that the proposed objective function approximators have high expressiveness as universal approximators,
and the \textit{global minimizer} can be obtained by \textit{a single convex optimization}.
As a realization of the proposed method,
extended PLSE (EPLSE) network is proposed by using the PLSE network as the PCM with slight modification.
The proposed method is demonstrated and compared with other existing approximators
by numerical simulation of parameterized non-convex objective function approximation and learning-based nonlinear model predictive control.

The rest of this paper is organized as follows.
\autoref{sec:preliminaries} provides the preliminaries of this study including convex analysis, set-valued analysis, parametric optimization and amortized optimization, and universal approximators.
In \autoref{sec:main_results},
the PCM method is proposed with optimality analysis, universal approximation theorem with implementation guidelines.
EPLSE network is proposed as a realization of the given method.
In \autoref{sec:numerical_simulation},
numerical simulation is performed to demonstrate the characteristics and the performance of the proposed EPLSE network.
\autoref{sec:conclusion} concludes this study with discussion and future works.

%% file: sections/2_preliminaries.tex
\section{Preliminaries}
\label{sec:preliminaries}
\subsection{Convex analysis}
\textit{Parameterized convexity} is an extension of convexity for parametric optimization,
interpreted as decision making in~\cite{kimParameterizedConvexUniversal2022b}.
\begin{definition}[Parameterized convexity]
  A function $f:X \times U \to \mathbb{R}$ is said to be \textup{parameterized convex} if $f(x, \cdot)$ is convex for any $x \in X$.
\end{definition}
Given function $f: U \to \mathbb{R}$,
a function $g$ is said to be a \textit{convex minorant} of $f$ if
$g$ is convex and $g(u) \leq f(u)$, $\forall u \in U$.
The \textit{greatest convex minorant} is defined as follows~\cite{abramsonConvexMinorantsRandom2011}.
\begin{definition}[Greatest convex minorant (GCM)]
  Given function $f:U \to \mathbb{R}$,
  \textup{the greatest convex minorant} of $f$ is denoted by
  $\textup{conv}f(u) :=\sup_{g \in G} g(u)$ where $G$ is the set of convex minorants of $f$.
\end{definition}
Given function $f: X \times U \to \mathbb{R}$,
a function $g$ is said to be a \textit{parameterized convex minorant} of $f$ if
$g$ is parameterized convex and $g(x, u) \leq f(x, u)$, $\forall (x, u) \in X \times U$.
The \textit{parameterized greatest convex minorant} is defined as an extension of greatest convex minorant for parametric optimization.
\begin{definition}[Parameterized greatest convex minorant (PGCM)]
  Given function $f:X \times U \to \mathbb{R}$,
  the parameterized greatest convex minorant of $f$ is denoted by
  $\textup{pconv}f(x, u) := \textup{conv}f_{x}(u) $ where $f_{x}(u) := f(x, u)$ for given $x \in X$.
\end{definition}

\subsection{Set-valued analysis}
A function $f:X \to Y$ is said to be a \textit{multivalued} function (or set-valued function, correspondence, etc.) if $f(x) \subset Y, \forall x \in X$.
Ordinary functions are referred to as \textit{single-valued} in set-valued analysis,
that is, $f(x) = \{y\}$ for $y \in Y$.
The graph of multivalued function $f: X \to Y$ is denoted by
$\text{Graph}(f) := \{ (x, y) \in X \times Y \vert y \in f(x) \}$.
A multivalued function $f: X \to Y$ is said to be upper hemicontinuous (u.h.c.) at $x_{0}$ if,
for any open neighborhood $V$ of $f(x_{0})$, there exists a neighborhood $U$ of $x_{0}$ such that $f(x) \subset V, \forall x \in U$~\cite{aubinSetValuedAnalysis2009}.

Given multivalued function $f: X \to Y$ and $\epsilon > 0$,
a (continuous) single-valued function $g: X \to Y$ is said to be a \textit{(continuous) approximate selection}
(or, simply, an \textit{$\epsilon$-selection}) of $f$ if $\text{Graph}(g) \subset B(\text{Graph}(f), \epsilon) := \cup_{(x, y) \in \text{Graph}(f)}B((x, y), \epsilon)$,
where $B(z, \epsilon) := \{w \vert \lVert w - z \rVert \leq \epsilon \}$ denotes the ball around $z \in \mathbb{R}^{n}$ with radius $\epsilon > 0$.

\subsection{Parametric optimization and amortized optimization}
Given objective function $f: X \times U \to \mathbb{R}$,
\textit{parametric optimization} finds a minimizer $u^{\star} \in U$ minimizing the objective function $f(x, \cdot)$ for given parameter $x \in X$.
The parametric optimization can be used to model a wide range of engineering applications, including the following:
\begin{itemize}
  \item (Reinforcement learning~\cite{suttonReinforcementLearningIntroduction2018}) Given state $s \in S$, find an optimal action $a^{\star} \in A$ minimizing state-action value function $Q(s, \cdot)$ (often referred to as Q-function).
  \item (Optimal control~\cite{liberzonCalculusVariationsOptimal2012}) Given state $x \in X$, find an optimal control input $u^{\star} \in U$ minimizing Hamiltonian $H(x, \cdot)$.
  \item (Model predictive control~\cite{camachoModelPredictiveControl2007}) Given state $x \in X$, find an optimal control input sequence $(u^{1}, \ldots, u^{N})^{\star} \in U \times \ldots \times U$ minimizing $N$-horizon objective function $J(x, \cdot, \ldots, \cdot)$, and then apply $u^{1}$ for each time step.
\end{itemize}
When certain related functions such as the minimizer function
and the objective function are learned to perform the parametric optimization (quickly) in operation,
it is referred to as \textit{amortized optimization}.

In this study,
let us suppose that
the parameter set $X \subset \mathbb{R}^{n}$ is compact,
the optimization variable set $U \subset \mathbb{R}^{m}$ is convex compact,
and the objective function $f: X \times U \to \mathbb{R}$ is continuous.
Additionally,
the following regularity condition is assumed:
\begin{assumption}
  \label{assumption:regularity_condition}
  Given a continuous objective function $f:X \times U \to \mathbb{R}$,
  for all $\epsilon > 0$,
  the multivalued minimizer set function $U^{\star}: X \to U$ admits a single-valued continuous approximate selection
  $h_{\epsilon}: X \to U$ where $U^{\star}(x) := \argmin_{u \in U}f(x, u)$, $\forall x \in X$,
  that is,
  $\text{Graph}(h_{\epsilon}) \subset B(\text{Graph}(U^{\star}), \epsilon)$.
\end{assumption}
That is, $\text{Graph}(h_{\epsilon}) \subset B(\text{Graph}(U^{\star}), \epsilon)$
implies that given $(x, u) \in X \times U$, there exists $x_{0} \in X$ and $u_{0} \in U^{\star}(x_{0})$
such that $\lVert(x_{0}, u_{0}) - (x, h_{\epsilon}(x)) \rVert < \epsilon$~\cite{aubinSetValuedAnalysis2009}.
It should be noted that this regularity condition is as mild as the utilization of a single-valued approximator for minimizer function approximation in amortized optimization.
In reinforcement learning,
this can be viewed as if the given problem admist a policy approximator.

\subsection{Existing universal approximators and problem formulation}
The universal approximation theorem (UAT) states that a certain approximator is capable of approximating a class of functions with arbitrary precision on a compact set,
and in this case, the function approximator is said to be a \textit{universal approximator}.
For example,
the feed-forward neural network (FNN) is a representative example of universal approximators for continuous functions~\cite{pinkusApproximationTheoryMLP1999,hornikMultilayerFeedforwardNetworks1989}:
given continuous function $f: X \to \mathbb{R}^{m}$ defined on a compact set $X \subset \mathbb{R}^{n}$,
for any $\epsilon > 0$,
there exists a network structure and parameters of FNN $\hat{f}$ such that $\lVert \hat{f} - f \rVert_{\infty} < \epsilon$.
If a universal approximator preserves its shape, it is said to be \textit{shape-preserving universal approximator}.

Max-affine (MA) and log-sum-exp (LSE) networks are shape-preserving universal approximators for convex functions~\cite{calafioreLogSumExpNeuralNetworks2020}.
The MA network can be written as
\begin{equation}
  \label{eq:MA}
  f^{\text{MA}}(z) = \max_{1 \leq i \leq I} \left(
      \langle a_{i}, z \rangle + b_{i}
  \right),
\end{equation}
where $a_{i} \in \mathbb{R}^{n}$ and $b_{i} \in \mathbb{R}$ for $i \in \{1, \ldots, I\}$ are network parameters.
The LSE network is a smoothed version of the MA network that replaces the max operator with the LSE operator.
The corresponding LSE network can be written as
\begin{equation}
  \label{eq:lse}
  f^{\text{LSE}}(z) = T \log \left(
    \sum_{i=1}^{I} \exp \left(
      \frac{\langle a_{i}, z \rangle + b_{i}}{T}
    \right)
  \right),
\end{equation}
where $T \in \mathbb{R}_{> 0}$ is referred to as \textit{temperature}~\cite{calafioreLogSumExpNeuralNetworks2020}.
The variable $z$ can be replaced as $z = [x^{\intercal}, u^{\intercal}]^{\intercal}$ with parameter $x$ and optimization variable $u$
for parametric optimization.
The resulting optimization problem is convex optimization because the projection preserves convexity~\cite{boydConvexOptimization2004,calafioreEfficientModelFreeQFactor2020}.

Parameterized MA (PMA) and parameterized LSE (PLSE) networks
are extended versions of MA and LSE networks, respectively,
for objective function approximation in amortized optimization, referred to as decision-making in~\cite{kimParameterizedConvexUniversal2022b}.
The PMA and PLSE networks are shape-preserving universal approximators for \textit{parameterized convex} functions,
hence they can cover a more general class of functions than MA and LSE networks.
The PLSE network can be written as
\begin{equation}
  \label{eq:plse}
  f^{\text{PLSE}}(x, u) = T \log \left(
    \sum_{i=1}^{I} \exp \left(
      \frac{\langle a_{\alpha_{i}}(x), u \rangle + b_{\beta_{i}}(x)}{T}
    \right)
  \right),
\end{equation}
where $a_{\alpha_{i}}: X \to \mathbb{R}^{m}$ and $b_{\beta_{i}}: X \to \mathbb{R}$
are the shape-preserving universal approximators for continuous functions with network parameters $\alpha_{i}$ and $\beta_{i}$,
respectively,
for $i \in \{1, \ldots, I\}$.
(P)MA and (P)LSE networks can obtain a global minimizer reliably and quickly by exploiting convexity~\cite{boydConvexOptimization2004,liuSurveyConvexOptimization2017}.
However,
the shape itself may restrict the class of functions to be approximated.

The difference of LSE (DLSE) network was proposed
to balance the trade-off of expressiveness and shape,
which is highly expressive and exploits convexity~\cite{calafioreUniversalApproximationResult2020}.
That is,
the DLSE network is a universal approximator for continuous functions,
and owing to difference of convex algorithms (DCA),
the DLSE network can obtain a local minimizer with global convergence and linear convergence.
The DLSE network can be written as
\begin{equation}
  \label{eq:dlse}
  f^{\text{DLSE}}(z) = f^{\text{LSE}}_{1}(z) - f^{\text{LSE}}_{2}(z),
\end{equation}
where $f^{\text{LSE}}_{i}$ are the LSE networks for $i \in \{1, 2\}$.
The same temperatures are typically assigned to the LSE networks of
the DLSE network~\cite{calafioreUniversalApproximationResult2020}.

In short,
for the objective function approximation method in amortized optimization,
approximators with high expressiveness, such as FNNs,
are usually non-convex and suffer from local optimality
without good convergence properties.
Some shape-preserving universal approximators exploiting convexity, such as the MA, LSE, PMA, and PLSE networks,
can avoid local optimality.
However,
they can approximate objective functions only with the same shape.
DLSE has high expressiveness with better convergence properties but still suffers from local optimality with more computational cost for the iterations of convex optimization.
The goal of this study is to propose a new objective function approximator for the objective function approximation method in amortized optimization
such that the objective function approximator has high expressiveness with guarantee of obtaining a global minimizer reliably and quickly.
More precisely,
the objective of this study is to propose a method to design universal approximators for continuous functions,
which can obtain a global minimizer by a single convex optimization in parametric optimization settings.

%% file: sections/3_main_results.tex
\section{Main results}
\label{sec:main_results}
In this section,
the parameterized convex minorant (PCM) method is proposed.
The PCM method is directed to find a minimizer of a highly expressive objective function approximator by a single convex optimization.

In the PCM method,
an objective function approximator $\hat{f}:X \times U \to \mathbb{R}$ is expressed as follows,
\begin{equation}
  \label{eq:pcm_method}
  \begin{split}
    \hat{f}(x, u)
  &:= f^{\text{PCM}}(x, u) + f^{\text{gap}}(x, u)
  \\
  f^{\text{gap}}(x, u)
  &:= \max\left(0, f^{\text{NN}}(x, u) - f^{\text{NN}}(x, \hat{u}^{\star}(x))\right) \geq 0,
  \end{split}
\end{equation}
where $f^{\text{PCM}}:X \times U \to \mathbb{R}$ is a PCM of the objective function approximator $\hat{f}$ such that
$f^{\text{PCM}}$ is a shape-preserving universal approximator for parameterized convex continuous functions.
A single-valued minimizer function $\hat{u}^{\star}:X \to U$ of the PCM
is utilized such that $\hat{u}^{\star}(x) \in \argmin_{u \in U} f^{\text{PCM}}(x, u) $, $\forall x \in X$.
The auxiliary approximator $f^{\text{NN}}:X \times U \to \mathbb{R}$ in the gap function $f^{\text{gap}}$ is a shape-preserving universal approximator for continuous functions.

\autoref{fig:concept} visualizes the characteristics of different approximators for the objective function approximation method in amortized optimization.
As seen in \autoref{fig:concept},
approximators with high expressiveness and local optimality may fail to find a global minimizer,
for example, FNN and DLSE network.
Approximators with global optimality and restrictive shapes,
such as PLSE networks,
may not suffer from local optimality, however,
these may fail to approximate the objective function with general shapes and may result in a low-quality approximate minimizer.
Unlike existing approximators,
in the proposed PCM method,
the objective function approximator has high expressiveness
and can find a global minimizer with a single convex optimization by using the PCM.

\subsection{Optimality}
In this section,
the optimality of the proposed objective function approximator is investigated.
The following theorem describes that the single-valued minimizer function of the PCM also attains the global minimum of the objective function approximator.
Thus,
a global minimizer of the proposed approximator can be found by a single convex optimization.
\begin{theorem}
  \label{thm:pcm_minimizer}
  In \eqref{eq:pcm_method},
  the single-valued minimizer function $\hat{u}^{\star}: X \to U$ of the parameterized convex minorant $f^{\text{PCM}}$
  also attains the minimum of the objective function approximator $\hat{f}$,
  i.e., $\hat{u}^{\star}(x) \in \argmin_{u \in U} \hat{f}(x, u)$, $\forall x \in X$.
\end{theorem}
\begin{proof}
  It is straightforward from \eqref{eq:pcm_method} that
  \begin{equation}
    \hat{f}(x, u)
    \geq f^{\text{PCM}}(x, u)
    \geq \min_{u' \in U} f^{\text{PCM}}(x, u'),
  \end{equation}
 $\forall (x, u) \in X \times U$,
 and for all $x \in X$,
 $\hat{f}(x, \hat{u}^{\star}(x))
 = f^{\text{PCM}}(x, \hat{u}^{\star}(x))
 = \min_{u \in U} f^{\text{PCM}}(x, u)$,
 which concludes the proof.
\end{proof}

The following theorem supports that the accurate objective function approximation
implies sub-optimality of the approximate minimizer obtained from the objective function approximator.
\begin{theorem}[Optimality of the approximate minimizer]
  \label{thm:optimality}
  Given objective function $f: X \times U \to \mathbb{R}$,
  $\epsilon > 0$,
  let us suppose that there exists an objective function approximator $\hat{f}: X \times U \to \mathbb{R}$
  such that $\lVert \hat{f} - f \rVert_{\infty} < \epsilon$.
  Then, for any $x \in X$,
  $f(x, \hat{u}^{\star}) < 2 \epsilon + \min_{u \in U}f(x, u)$
  for any approximate minimizer $\hat{u}^{\star}$
  where $\hat{u}^{\star} \in \argmin_{u \in U} \hat{f}(x, u) $.
\end{theorem}
\begin{proof}
  The following proof is borrowed from~\cite[Secion III.C]{calafioreEfficientModelFreeQFactor2020}
  and presented here for completeness.
  It can be deduced from the assumption in \autoref{thm:optimality} that
  for all $x \in X$,
  $ \left(\hat{f}(x, u^{\star}) - f(x, u^{\star})\right) - \left(\hat{f}(x, \hat{u}^{\star}) - f(x, \hat{u}^{\star})\right) < \epsilon + \epsilon = 2 \epsilon$
  for any $u^{\star} \in \argmin_{u \in U} f(x, u) $.
  Because $\hat{u}^{\star}$ is the minimizer of $\hat{f}(x, \cdot)$,
  this implies that
  $f(x, \hat{u}^{\star}) < 2 \epsilon + f(x, u^{\star}) + \left(\hat{f}(x, \hat{u}^{\star}) - \hat{f}(x, u^{\star})\right) \leq 2 \epsilon + f(x, u^{\star}) + 0 = 2 \epsilon + \min_{u \in U} f(x, u) $,
  which concludes the proof.
\end{proof}
Combining \autoref{thm:pcm_minimizer} and \autoref{thm:optimality} implies in the proposed method that
an approximate global minimizer can be obtained from the objective function approximator \eqref{eq:pcm_method} by minimizing the PCM instead.

\subsection{Universal approximation theorem}
In this section,
the universal approximation theorem of the objective function approximator
is established in the PCM method.

The following theorem is the universal approximation theorem of the objective function approximator in the proposed PCM method.
\begin{theorem}[Universal approximation theorem in the PCM method]
  \label{thm:universal_approximation_theorem}
  Given continuous function $f: X \times U \to \mathbb{R}$,
  for any $\epsilon > 0$,
  there exists an approximator $\hat{f}: X \times U \to \mathbb{R}$ in the form of \eqref{eq:pcm_method}
  such that $\lVert \hat{f} - f \rVert_{\infty} < \epsilon$.
\end{theorem}
\begin{proof}
  See \hyperref[sec:proof_of_uat]{Appendix A}.
\end{proof}
\autoref{thm:universal_approximation_theorem} supports that
the proposed objective function approximator has high expressiveness as a universal approximator for continuous functions.

\subsection{Implementation guidelines}
Two problems remain when implementing the objective function approximator in the proposed PCM method:
i) How can the single-valued minimizer function $\hat{u}^{\star}: X \to U$ of the PCM in \eqref{eq:pcm_method} be found?
ii) How can the objective function approximator $\hat{f}$ in \eqref{eq:pcm_method} be trained?
In this study,
to realize the PCM method,
the PCM and gap function are explicitly parameterized as follows:
\begin{equation}
  \begin{split}
    \hat{f}_{\theta}(x, u)
  &= f_{\theta_{1}}^{\text{PCM}}(x, u)
  \\
  &+ \max\left(0, f_{\theta_{2}}^{\text{NN}}(x, u) - f_{\theta_{2}}^{\text{NN}}(x, \hat{u}^{\star}(x; \theta_{1}))\right),
  \end{split}
\end{equation}
where $\theta_{1}$ and $\theta_{2}$ are the network parameters of the PCM $f^{\text{PCM}}$ and the auxiliary function $f^{\text{NN}}$, respectively.
Therefore,
the network parameters of the objective function approximator can be expressed as
$\theta = (\theta_{1}, \theta_{2})$.
In this study,
the parameters $x$ and $\theta$ are distinguished as parameter and network parameters, respectively.

To find the single-valued minimizer function $\hat{u}^{\star}$ in \eqref{eq:pcm_method},
any convex optimization solver minimizing the PCM can be used,
for example, ECOS~\cite{domahidiECOSSOCPSolver2013}.
In this case,
the value of the minimizer depends on the network parameter of the PCM $f^{\text{PCM}}$,
and the minimizer can explicitly be expressed as $\hat{u}^{\star}(x; \theta_{1})$ for given parameter $x \in X$.

In deep learning,
gradient-based methods such as ADAM~\cite{kingmaAdamMethodStochastic2017} are widely used to train approximators.
The gradient of the output of the objective function approximator in \eqref{eq:pcm_method} with respect to the network parameters can be obtained by using the following gradients with chain rules:
$\nabla_{\theta_{1}} f_{\theta_{1}}^{\text{PCM}}(x, u)$,
$\nabla_{\theta_{2}} f_{\theta_{2}}^{\text{NN}}(x, u)$,
$\nabla_{u} f_{\theta_{2}}^{\text{NN}}(x, u)$,
and $\nabla_{\theta_{1}} \hat{u}^{\star}(x; \theta_{1}) $.
Other gradients can be calculated via the backpropagation using automatic differentiation tools.
The challenging part is the gradient of the minimizer function with respect to the network parameter of the PCM, $\nabla_{\theta_{1}} \hat{u}^{\star}(x; \theta_{1})$.
Recent advances in automatic differentiation enable us to calculate $\nabla_{\theta_{1}} \hat{u}^{\star}(x; \theta_{1})$, including differentiable convex optimization layers~\cite{agrawalDifferentiableConvexOptimization2019} and automatic implicit differentiation~\cite{blondelEfficientModularImplicit2022}.
For example,
differentiable convex optimization layers require that the corresponding optimization problem be expressed in disciplined parameterized programming (DPP) and that the minimizer is unique.
It is challenging to guarantee that the PCM has a unique minimizer for any parameter $x \in X$.
In this study, a modified PLSE network, \textit{PLSE+} network, is proposed to mitigate this challenge.
Given PLSE network $f^{\text{PLSE}}$ in \eqref{eq:plse},
the corresponding PLSE+ network $f^{\text{PLSE+}}$ nullifies $a_{\alpha_{1}}$, i.e., $a_{\alpha_{1}}(x) \equiv 0$.
This modification makes the network be \textit{likely strictly convex} in statistical settings~\cite{nielsenMonteCarloInformation2018},
which implies the uniqueness of the minimizer.
Additionally,
it is straightforward that the minimization of the PLSE(+) network is DPP
because the parameterized log-sum-exp problem is DPP, similar to~\cite[Theorem 1]{kimOfflineDifferentiableQlearning2022}.
To realize the PCM method,
the \textit{extended PLSE} (EPLSE) network is proposed as follows,
\begin{equation}
  \begin{split}
    \label{eq:eplse}
    f^{\text{EPLSE}}(x, u)
  &= f^{\text{PLSE+}}(x, u)
  \\
  &+ \max\left(0, f^{\text{NN}}(x, u) - f^{\text{NN}}(x, \hat{u}^{\star}(x))\right).
  \end{split}
\end{equation}

It is straightforward to show that PLSE+ network is parameterized convex continuous.
The PLSE+ network is a shape-preserving universal approximator for parameterized convex continuous functions by the following theorem.
\begin{theorem}[Universal approximation theorem of PLSE+ network]
  \label{thm:uat_of_plse_plus}
  Given function $f:X \times U \to \mathbb{R}$ parameterized convex continuous,
  for all $\epsilon > 0$,
  there exists a PLSE+ network $f^{\textup{PLSE+}}$ such that
  $\lVert f^{\textup{PLSE+}} - f\rVert_{\infty} < \epsilon $.
\end{theorem}
\begin{proof}
  See \hyperref[sec:proof_of_uat_plseplus]{Appendix B}.
\end{proof}

%% file: sections/4_numerical_simulation.tex
\section{Numerical simulation}
\label{sec:numerical_simulation}
In this section,
the proposed objective function approximator is compared to other approximators
for amortized optimization with applications to
i) parameterized non-convex objective function approximation
and ii) learning-based nonlinear model predictive control (MPC).
Data are randomly sampled and split into training, validation, and test datasets.
For numerical simulation,
the hyperparameters of each approximator are set as $I = 20$ and $T = 1$ for DLSE, PLSE, and the PLSE+ of EPLSE networks,
and hidden layer nodes of $(64, 64)$ for the FNN and the auxiliary universal approximator $f^{\text{NN}}$ of the EPLSE network in \eqref{eq:eplse}.
Simulations were performed on a desktop with an AMD Ryzen 9 5900X.
Approximators are trained with the training data by mini-batch supervised learning by using optimizer Adam~\cite{kingmaAdamMethodStochastic2017}.
The training epochs are $200$ for parameterized non-convex objective function approximation and learning-based nonlinear MPC.
For the convex optimization of PLSE, DLSE, and EPLSE,
ECOS is used~\cite{domahidiECOSSOCPSolver2013}.
For the non-convex optimization of the FNN,
an interior-point Newton method is used~\cite{kmogensenOptimMathematicalOptimization2018}.
At each epoch,
the loss of the approximator is evaluated over validation data, and the approximator with the smallest validation loss is saved as the best approximator.
Test datasets are used to evaluate the best approximators after training.
The code is publicly available\footnote{https://github.com/JinraeKim/PCMAO}.

\subsection{Case 1: Parameterized non-convex objective function approximation}
For the parameterized non-convex objective function approximation,
an example~\cite[Example 3]{calafioreUniversalApproximationResult2020} is modified as the following function $f:X \times U \to \mathbb{R}$,
\begin{equation}
  f(x, u) = x^{2} + u^{2} + \sin(2 \pi u),
\end{equation}
where the parameter and optimization variable sets are set as $X = [-1, 1]$ and $U = [-1, 1]$, respectively.
The objective function $f$ is parameterized non-convex.
Simulation settings are summarized in \autoref{table:settings}.

\begin{table}[!t]
    \centering
    \begin{tabular}{l|c|c|c}
        \hline
        & $\#$ of data$^{*}$ & learning rate (lr) & dimensions of ($x, u$)
        \\   
        \hline
        \hline
      \texttt{Case 1} & $2\text{,}000$ & $10^{-3}$ (or $10^{-1}$)$^{**}$ & $(1, 1)$
        \\        
        \hline
      \texttt{Case 2} & $10\text{,}000$ & $10^{-3}$ & $(4, 5)$
        \\
        \hline
        \hline
        \multicolumn{4}{l}{%
          \begin{minipage}{6cm}%
            $^{*}$: train:valid:test = 0.7:0.2:0.1, mini-batch size=16
            \\
            $^{**}$: DLSE with lr of $10^{-3}$ showed too slow training progress.
          \end{minipage}%
        }\\
    \end{tabular}
    \caption{Simulation settings}
    \label{table:settings}
\end{table}

\begin{table*}[!t]
    \centering
    \begin{tabular}{l|c|c|c|c}
        \hline
        & FNN & PLSE & DLSE & EPLSE (proposed)
        \\   
        \hline
        \hline
      Mean minimizer errors$^{*}$ $\downarrow$ & $0.4155$ & $0.1207$ & $0.4080$ & $\textbf{0.0077}$
        \\        
        \hline
      Mean minimum value errors$^{**}$ $\downarrow$ & $0.3229$ & $0.8489$ & $0.3224$ & $\textbf{0.0180}$
        \\
        \hline
      Mean solve time [s] $\downarrow$ & $\textbf{0.0025}$ & $0.0030$ & $0.0080$ & $0.0029$
      \\
        \hline
        \hline
        \multicolumn{5}{l}{%
          \begin{minipage}{10cm}%
            $^{*}$: $l_{2}$-norm
            \\
            $^{**}$: absolute value
          \end{minipage}%
        }\\
    \end{tabular}
    \caption{Simulation results (Case 1)}
    \label{table:results_case1}
\end{table*}
\begin{table*}[!t]
    \centering
    \begin{tabular}{l|c|c|c|c|c}
        \hline
        & FNN & PLSE & DLSE & EPLSE (proposed) & linear MPC
        \\   
        \hline
        \hline
      $\lvert \phi(t_{f}) - \phi_{d} \rvert$ [deg] $\downarrow$ & $0.9345$ & $1.4485$ & $0.7311$ & $\textbf{0.1286}$ & $0.9034$
        \\
        \hline
      $\lvert \dot{\phi}(t_{f}) - \dot{\phi}_{d} \rvert$ [deg] $\downarrow$ & $5.8360$ & $0.0000$ & $0.6037$ & $0.0001$ & $0.0739$
        \\
        \hline
      Mean of $J(x, \hat{u}^{\star})$ $\downarrow$ & $0.0812$ & $\textbf{0.0760}$ & $0.0761$ & $0.0765$ & $0.0832$
        \\
        \hline
      Mean solve time [s] $\downarrow$ & $0.0151$ & $0.0034$ & $0.0802$ & $0.0062$ & $\textbf{0.0025}$
      \\
        \hline
        \hline
    \end{tabular}
    \caption{Simulation results (Case 2)}
    \label{table:results_case2}
\end{table*}

\begin{figure*}
  \centering
  \begin{subfigure}[b]{0.48\linewidth}
    \centering
    \includegraphics[width=1.0\linewidth]{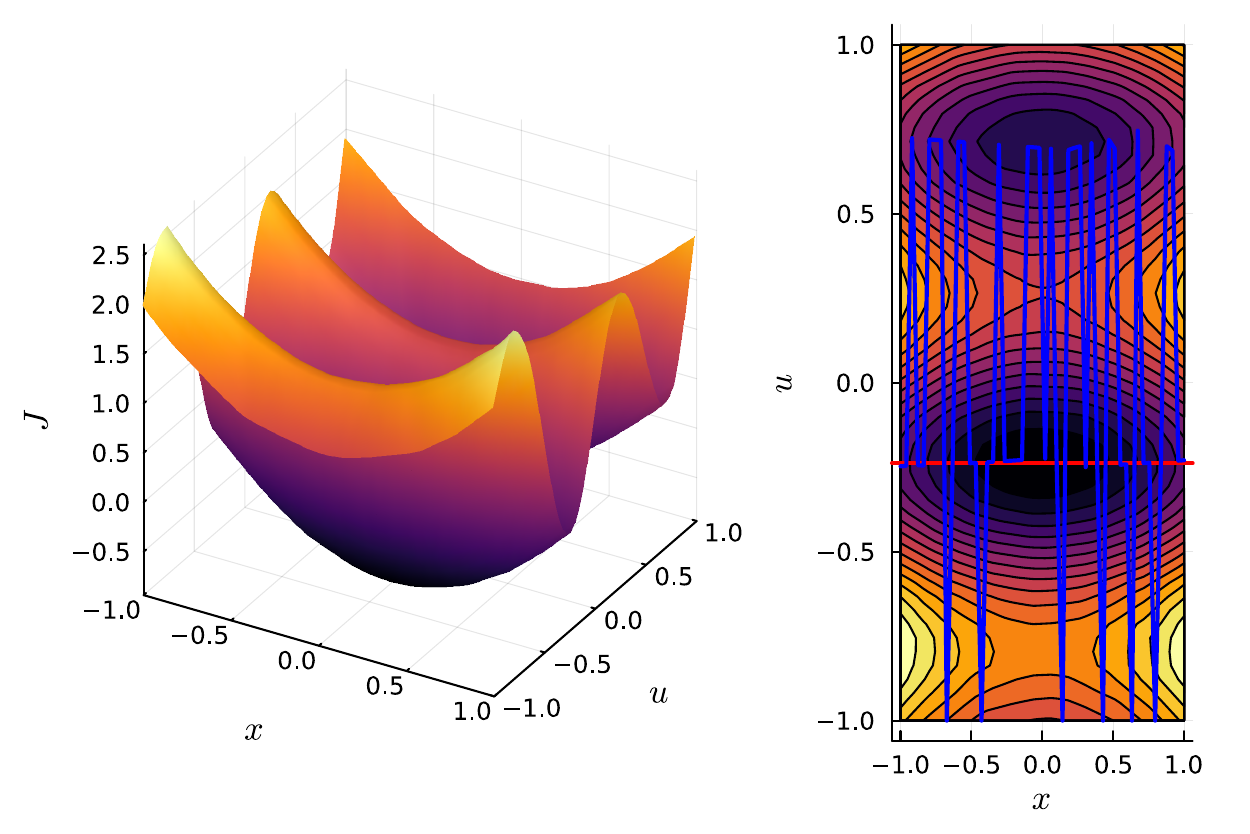}
    \caption{FNN}
  \end{subfigure}
  \begin{subfigure}[b]{0.48\linewidth}
    \centering
    \includegraphics[width=1.0\linewidth]{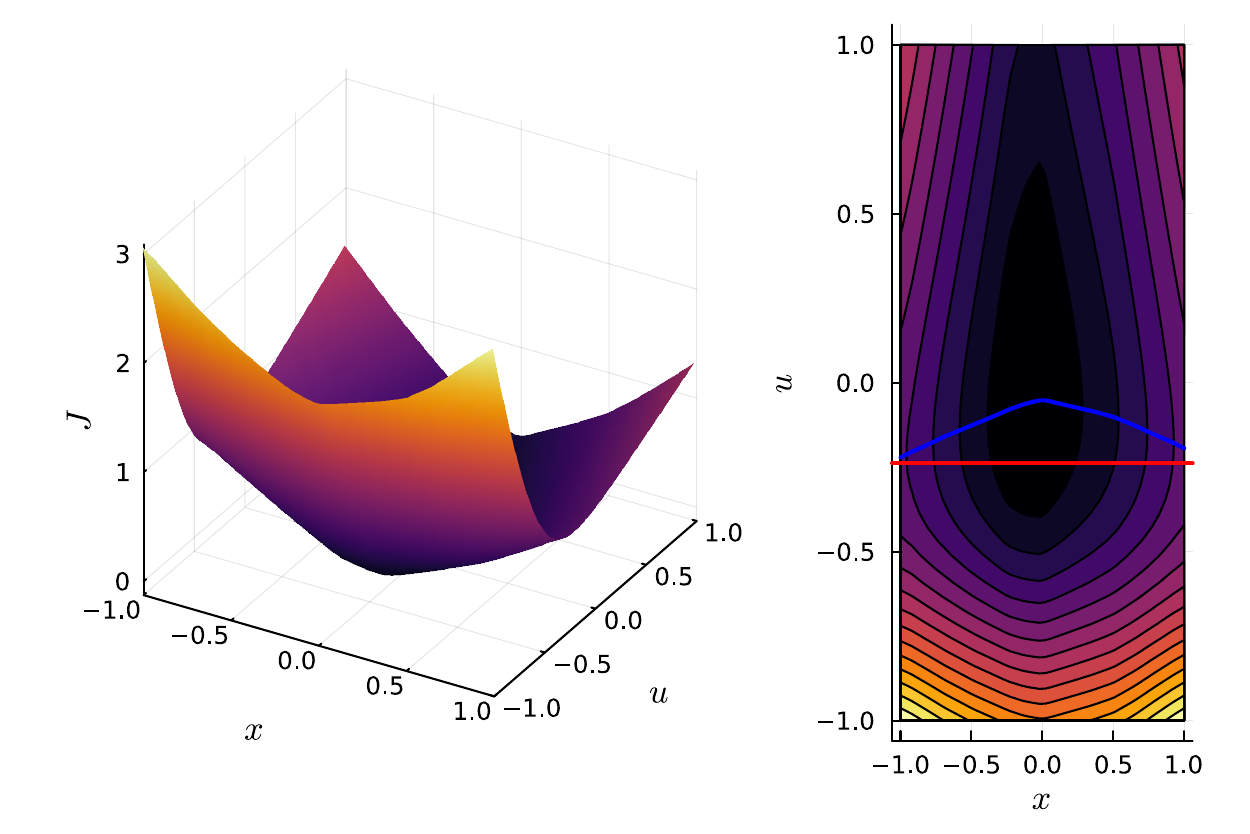}
    \caption{PLSE}
  \end{subfigure}
  \begin{subfigure}[b]{0.48\linewidth}
    \centering
    \includegraphics[width=1.0\linewidth]{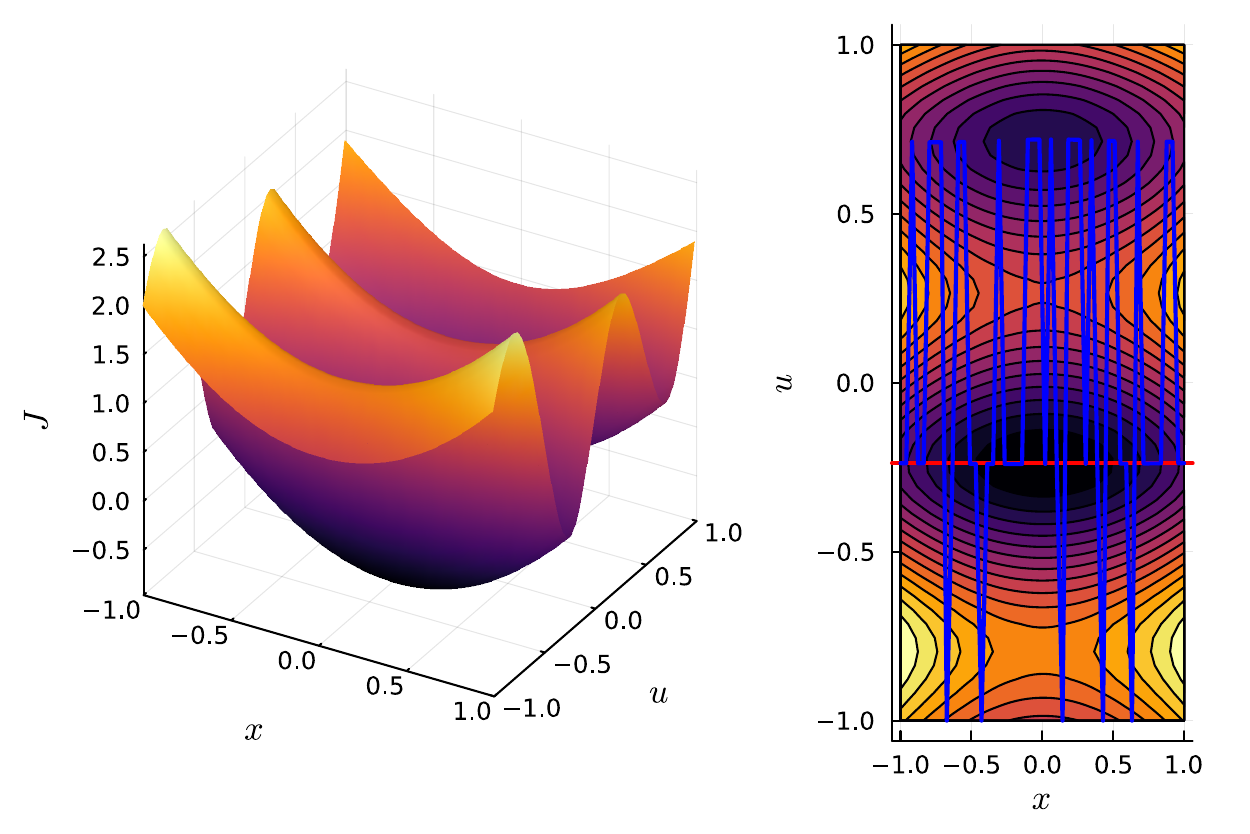}
    \caption{DLSE}
  \end{subfigure}
  \begin{subfigure}[b]{0.48\linewidth}
    \centering
    \includegraphics[width=1.0\linewidth]{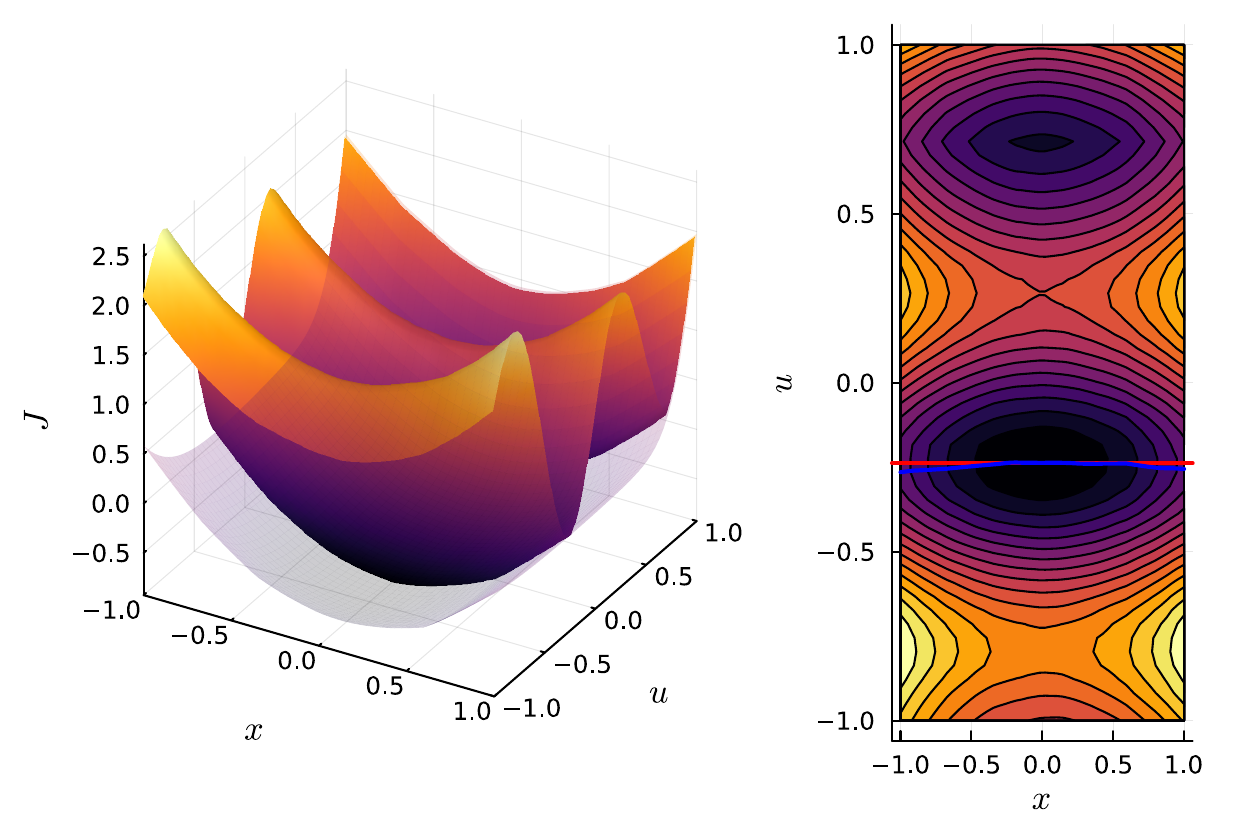}
    \caption{EPLSE (proposed) with PCM in gray}
  \end{subfigure}
  \begin{subfigure}[b]{0.48\linewidth}
    \centering
    \includegraphics[width=1.0\linewidth]{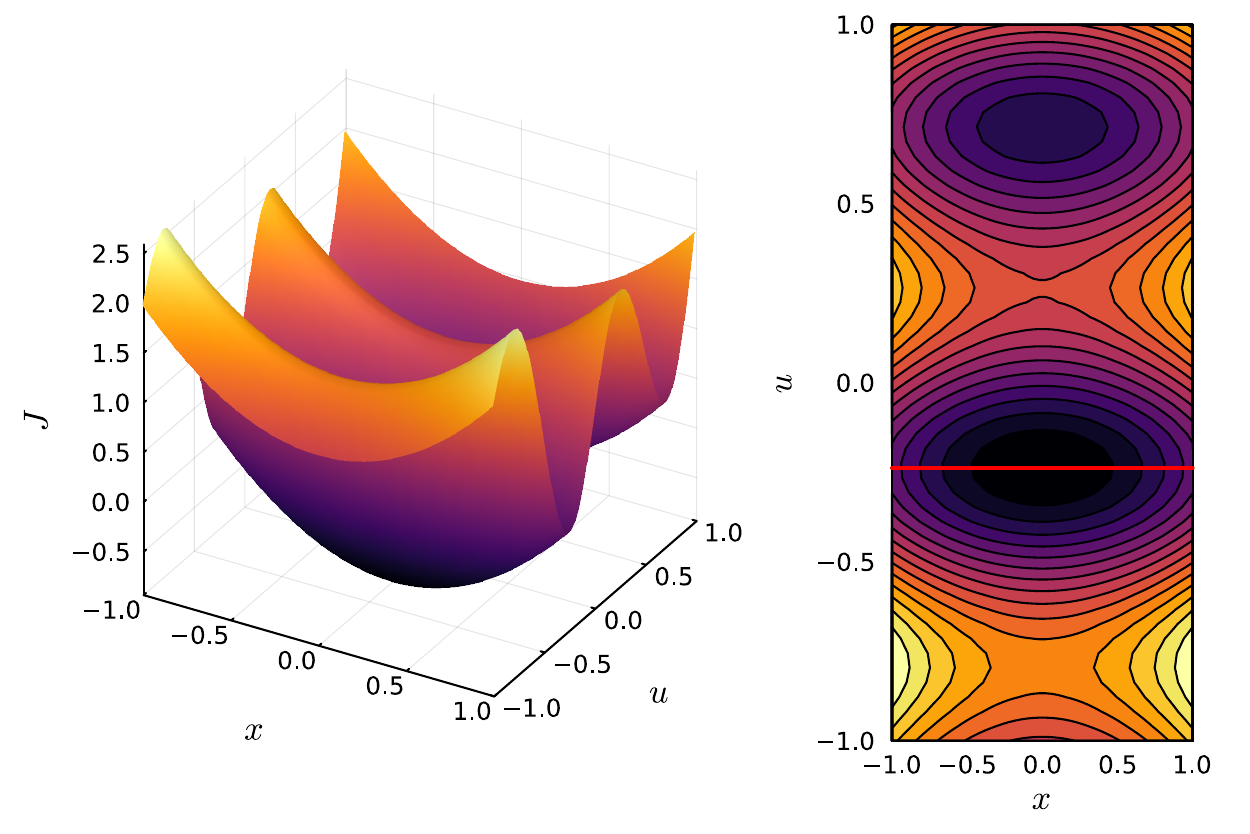}
    \caption{Target function}
  \end{subfigure}
  \caption{Surface and contour visualization (Case 1).
  The red line denotes the true minimizer.
The blue line denotes the approximate minimizer obtained from each approximator.}
  \label{fig:simple_func_results}
\end{figure*}

\autoref{fig:simple_func_results} shows the target objective function and approximation results of FNN, PLSE, DLSE, and the proposed approximator.
FNN and DLSE are continuous universal approximators,
and therefore,
they can approximate the target function well.
However, FNN and DLSE suffer from local optimality,
which results in chattering of the approximate minimizer depending on the parameter.
The approximate minimizer of the PLSE network does not show a chattering phenomenon.
However, PLSE is a shape-preserving universal approximator for parameterized convex functions,
and therefore,
PLSE poorly approximates the parameterized non-convex objective function,
resulting in low-quality optimization.
However,
the proposed approximator, EPLSE network, approximates the target objective function with high precision as well as performant approximate minimizer without chattering.
This result supports that the proposed approximator can approximate a parameterized non-convex objective function
and can retrieve the global minimizer reliably.
The minimizer errors and minimum value errors over the test dataset are summarized in \autoref{table:results_case1}.

\subsection{Case 2: Learning-based nonlinear model predictive control}
In this section,
the proposed method is demonstrated with a learning-based control application, nonlinear model predictive control (NMPC).
Given continuous-time dynamical system $\dot{\textbf{x}}(t) = f_{\text{cont}}(\textbf{x}(t), \textbf{u}(t))$,
the corresponding discrete-time dynamical system using zero-order-hold scheme with time step $\Delta t$ can be written as
$\textbf{x}_{n+1} = f_{\text{disc}}(\textbf{x}_{n}, \textbf{u}_{n})$ where $(\cdot)_{n} := (\cdot)(n \Delta t)$.
$\textbf{x}(t) \in \mathbb{X}$ and $\textbf{u}(t) \in \mathbb{U}$ denote the state and input at time $t$, respectively.

Given initial state $\textbf{x}_{0} \in \mathbb{X}_{0}$
and setpoint $\textbf{x}_{d} \in \mathbb{X}_{d}$ (with abuse of notation),
finite-horizon cost function can be defined as the objective function of parametric optimization as follows,
\begin{equation}
  \begin{split}
    f(x, u)
    &= (\textbf{x}_{N} - \textbf{x}_{d})^{\intercal} Q_{N} (\textbf{x}_{N} - \textbf{x}_{d})
    \\
    &+ \sum_{n=0}^{N-1} \left((\textbf{x}_{n} - \textbf{x}_{d})^{\intercal} Q (\textbf{x}_{n} - \textbf{x}_{d})
    + \textbf{u}_{n}^{\intercal} R \textbf{u}_{n}\right),
  \end{split}
\end{equation}
where parameter and optimization variable are set as
$x = [\textbf{x}_{0}^{\intercal}, \textbf{x}_{d}^{\intercal}]^{\intercal}$
and $u = [\textbf{u}_{0}^{\intercal}, \ldots, \textbf{u}_{N-1}^{\intercal}]^{\intercal}$, respectively.
The objective function $f$ is not parameterized convex in general due to the nonlinearity of the discrete-time dynamics $f_{\text{disc}}$.
The goal of learning-based NMPC is to find the approximate minimizer $u^{\star} = [(\textbf{u}_{0}^{\star})^{\intercal}, \ldots, (\textbf{u}_{N-1}^{\star})^{\intercal}]^{\intercal}$ minimizing the objective function $f(x, \cdot)$ for given parameter $x$,
and then apply the optimal input at first time step $\textbf{u}_{0}^{\star}$ at each time instant.

For the demonstration,
a wing-rock model for delta-wing aircraft is used~\cite{tarnFuzzyControlWing1993}.
The continuous-time dynamics of the wing-rock model can be written as
\begin{equation}
  \ddot{\phi} + \omega^{2} \phi = \mu_{1} \dot{\phi} + b_{1} \phi^{3} + \mu_{2} \phi^{2} \dot{\phi} + b_{2} \phi \dot{\phi}^{2} + \delta_{g},
\end{equation}
where $\phi(t)$ and $\delta_{g}(t)$ is the roll angle and generalized control surface deflection at time $t$,
and $b_{1}, b_{2}, \mu_{1}, \mu_{2}, \omega $ are constant parameters.
The state and input variables can be defined for state-space representation as
$\mathbf{x}(t) := [\phi(t), \dot{\phi}(t)]^{\intercal}$ and $\mathbf{u}(t) := \delta_{g}(t)$.
The initial state and setpoint sets are given as
$\mathbb{X}_{0} := [-25, 25] \text{(deg)} \times [-50, 50] \text{(deg/s)}$
and $\mathbb{X}_{d} := [-25, 25] \text{(deg)} \times \{0\} \text{(deg/s)}$,
that is, the parameter space is $X := \mathbb{X}_{0} \times \mathbb{X}_{d}$.
The input space is given as $\mathbb{U} = [-1.75, 1.75]$~\cite{tarnFuzzyControlWing1993},
and therefore,
the optimization variable space is $U := \mathbb{U}^{N}$.
The horizon and time step are set as $N = 5$ and $\Delta t = 0.1$s, respectively.

\autoref{fig:wingrock_setpoint} shows the simulation result of learning-based NMPC,
and \autoref{fig:wingrock_setpoint_inset} is a zoomed-in view of \autoref{fig:wingrock_setpoint}.
As a benchmark,
the result of a linear MPC constructed with a linearized wing-rock model around the origin is also shown for comparison.
The linear MPC is constructed with a known dynamic model for linearization
as well as equilibria with required inputs for setpoint tracking,
while other approximates do not have the knowledge of the model and equilibria.
In the simulation,
the initial state and desired setpoint are set as
$\textbf{x}_{0} = [10 \text{ (deg)}, 45 \text{ (deg/s)}]^{\intercal}$
and $\textbf{x}_{d} = [-25 \text{ (deg)}, 0 \text{ (deg/s)}]^{\intercal}$, respectively.
As seen in \autoref{fig:wingrock_setpoint},
the FNN shows a poor control performance with chattering in the control input due to the local optimality of the FNN.
The PLSE network shows a good response with relatively large steady-state error without chattering input.
As seen in \autoref{fig:wingrock_setpoint_inset}
DLSE network shows smaller steady-state error than that of PLSE.
However,
the DLSE network also shows slight chattering in the control input due to local optimality.
On the other hand,
the EPLSE network shows the fastest convergence and the smallest steady-state error without chattering in the control input.
Linear MPC shows a considerably slow convergent response,
and therefore,
the response does not converge to the setpoint within $t_{f} = 15$s.
PLSE, DLSE, and EPLSE networks show much faster convergence behavior (converged nearly at $t=5$s) compared to that of the linear MPC because of the consideration of nonlinear dynamics via data,
whereas linear MPC cannot properly reflect the growing nonlinearity far from the origin, which makes it slow to converge.
All cases do not violate input constraints because both learning-based and model-based methods can incorporate the input constraints into the optimization problem.
\autoref{table:results_case2} summarizes the quantitative results of Case 2 in terms of simulation and evaluation over the test dataset.
Although the PLSE network shows the smallest mean objective value evaluated over the test dataset,
the PLSE, DLSE, and EPLSE networks show similar levels of the mean objective.
In contrast,
FNN and linear MPC show relatively high mean objectives compared to
the others.
In terms of the mean solve time over the test dataset,
the PLSE network shows the fastest solve time compared to other approximators.
The mean solve time of PLSE and EPLSE networks are very small and similar to that of linear MPC (approximately $2 \sim 6$ms).
On the other hand,
the mean solve time of FNN and DLSE network are about $15$ms to $80$ms,
which may not be applicable for real-time applications.
Overall, the proposed EPLSE network shows real-time optimization with fast tracking to the setpoint
in the learning-based nonlinear MPC demonstrated with the wing-rock model.
\begin{figure}
  \centering
  \includegraphics[width=0.85\linewidth]{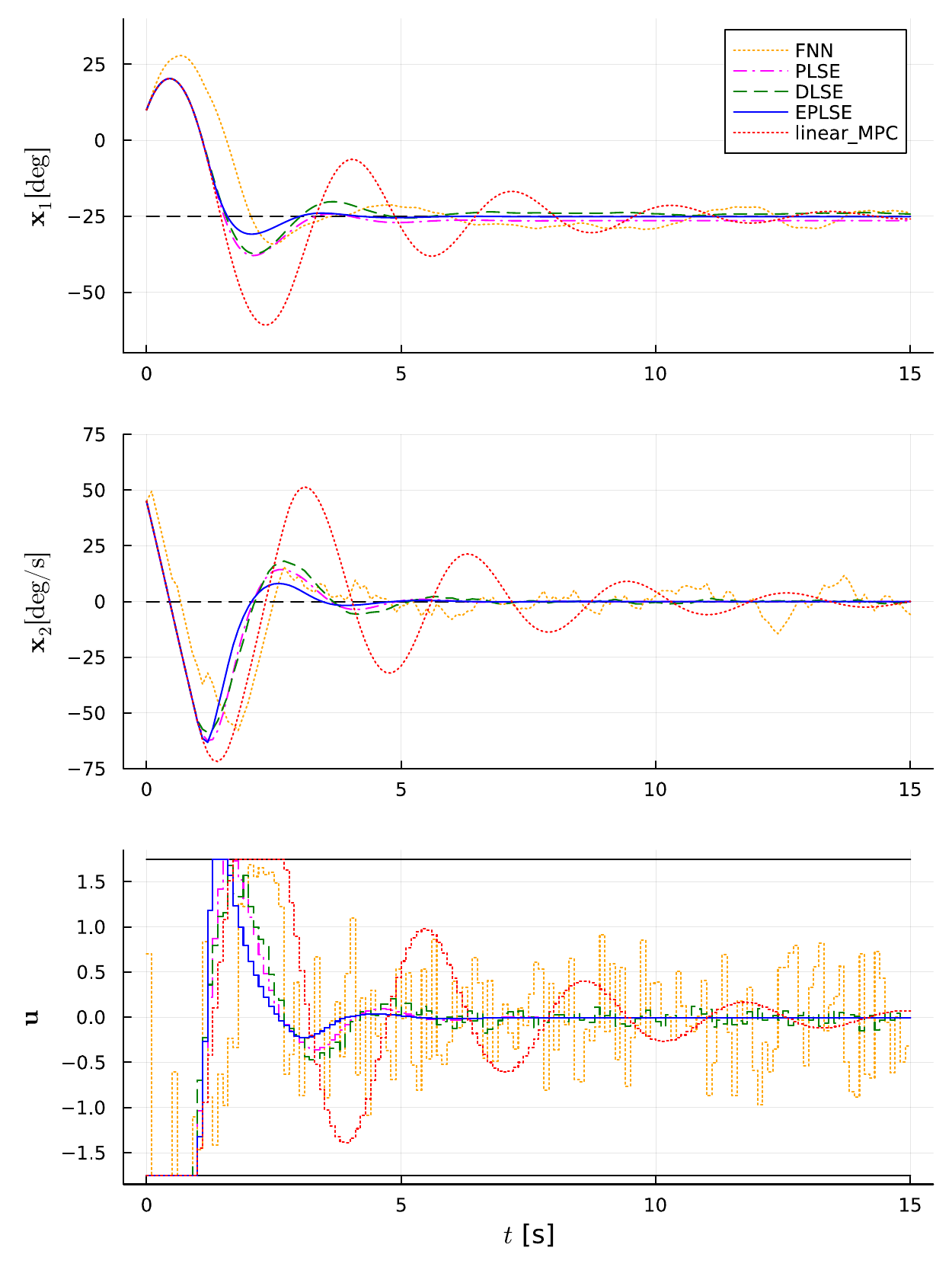}
  \caption{
    Response of learning-based NMPC (Case 2).
    Black dashed lines denote the given setpoint.
    Black solid lines denote the input limits.
  }
  \label{fig:wingrock_setpoint}
\end{figure}

\begin{figure}
  \centering
  \includegraphics[width=0.85\linewidth]{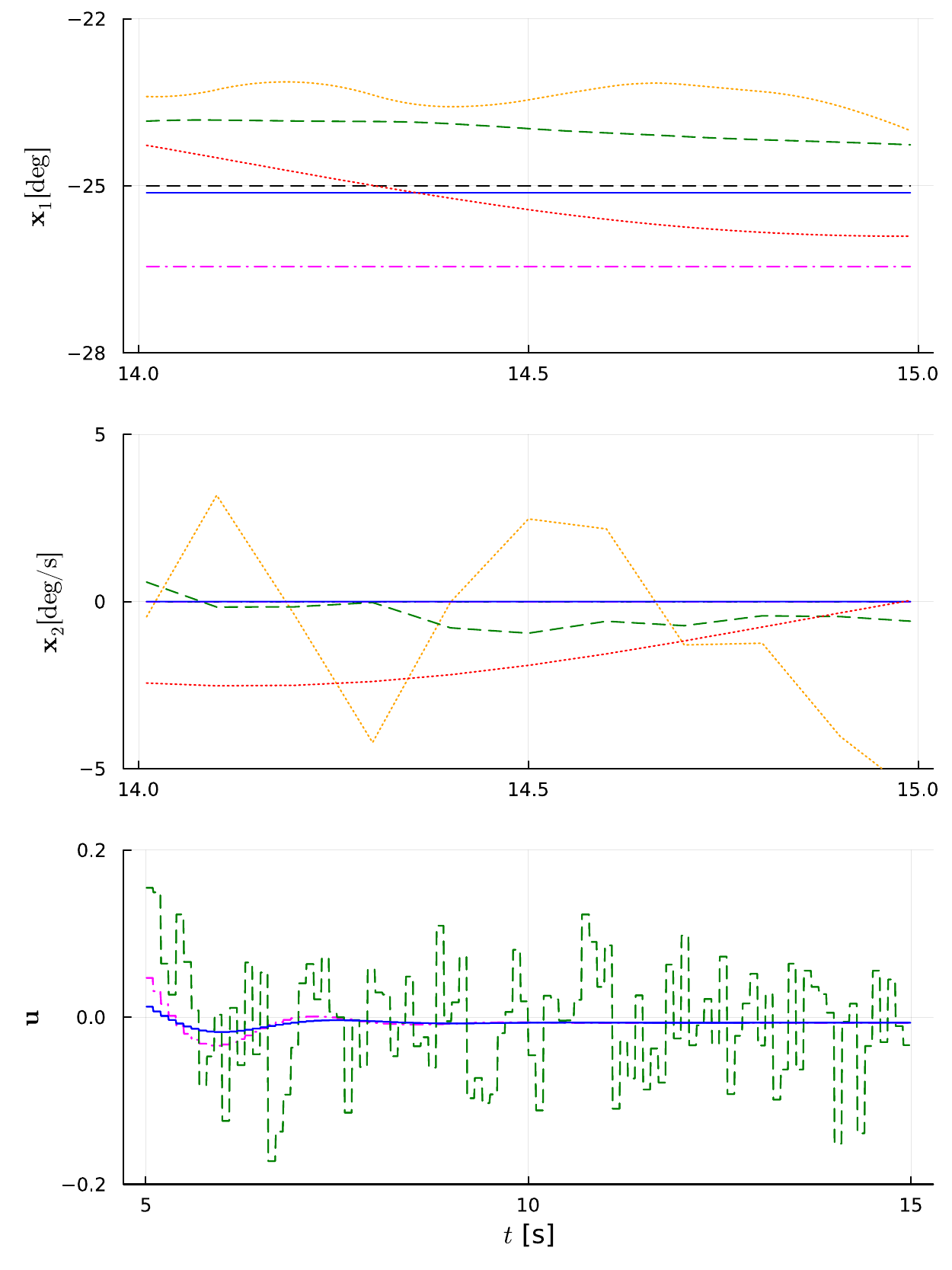}
  \caption{
    Zoomed-in view of \autoref{fig:wingrock_setpoint}.
    Each row may correspond to different time window.
  }
  \label{fig:wingrock_setpoint_inset}
\end{figure}

%% file: sections/5_conclusion.tex
\section{Conclusion}
\label{sec:conclusion}
In this study,
parameterized convex minorant (PCM) method
was proposed as a new approach to objective function approximation in amortized optimization.
In the proposed method,
an objective function approximator is constructed with a PCM and a nonnegative gap function.
Employing shape-preserving universal approximators for parameterized convex continuous and continuous functions as the PCM and the gap function, respectively,
it was shown that the single-valued minimizer function of the PCM attains the global minimum of the objective function approximator.
Moreover,
the objective function approximator is a universal approximator for continuous functions.
These imply that a global minimizer of the objective function approximator can be found by minimizing the PCM,
which costs only a single convex optimization.
To realize the proposed method,
extended parameterized log-sum-exp (EPLSE) network was proposed
by utilizing a modified parameterized log-sum-exp network as the PCM of the objective function approximator.
Numerical simulation results for parameterized non-convex objective function approximation and learning-based nonlinear model predictive control support that the EPLSE network can approximate parameterized non-convex objective functions and the minimizer can be obtained quickly and reliably.

The proposed method is promising in amortized optimization,
however, a limitation is expected:
It may take a long time in the proposed method to train the objective function approximator than to train others
because not only the minimization but also the evaluation of the objective function approximator requires a single convex optimization.
To avoid the aforementioned issue,
future work includes new training methodology for the PCM method to reduce the training time.

%% file: appendices/appendix_A.tex
\section{Proof of \autoref{thm:universal_approximation_theorem}}
\setcounter{equation}{0}  %
\renewcommand{\theequation}{\thesection.\arabic{equation}}  %
\label{sec:proof_of_uat}
Before the proof of the universal approximation theorem,
several lemmas are provided to describe basic characteristics of the parameterized greatest convex minorant (PGCM).
The following lemmas show that the PGCM of a continuous function is parameterized convex continuous.
\begin{lemma}
  \label{lemma:pconv_is_parameterized_convex}
  Given function $f: X \times U \to \mathbb{R}$, $\textup{pconv}f$ is parameterized convex.
\end{lemma}
\begin{proof}
  Given $x \in X$,
  let $f_{x}(u) := f(x, u)$.
  Let $G_{x}$ be the set of convex minorants of $f_{x}$.
  Then, $\forall g \in G_{x}$,
  \begin{equation}
    \begin{split}
      g(\lambda &u_{1} + (1-\lambda) u_{2})
      \\
                &\leq \lambda g(u_{1}) + (1-\lambda) g(u_{2})
                \\
                &\leq \lambda \textrm{conv}f_{x}(u_{1}) + (1-\lambda) \textrm{conv}f_{x}(u_{2}).
    \end{split}
  \end{equation}
  Taking the supremum yields
  \begin{equation}
    \begin{split}
      \textrm{conv}f_{x}(\lambda &u_{1} + (1-\lambda) u_{2})
      \\
                                 &\leq \lambda \textrm{conv}f_{x}(u_{1}) + (1-\lambda) \textrm{conv}f_{x}(u_{2}),
    \end{split}
  \end{equation}
  and replacing $\textrm{conv}f_{x}(\cdot)$ with $\textrm{pconv}f(x, \cdot)$ implies the parameterized convexity of $\textrm{pconv}f$.
\end{proof}

\begin{lemma}
  \label{lemma:pconv_is_continuous}
  Given continuous function $f: X \times U \to \mathbb{R}$, $\textup{pconv}f$ is continuous.
\end{lemma}
\begin{proof}
  Given $x \in X$,
  let $f_{x}(u) := f(x, u)$.
  Then, $\textrm{pconv}f(x, \cdot) = \textrm{conv}f_{x}(\cdot)$ is convex and finite, and therefore $\textrm{conv}f_{x}(\cdot)$ is continuous~\cite[Corollary 10.1.1]{rockafellarConvexAnalysis1970}.
  That is,
  given $(x_{0}, u_{0}) \in X \times U$,
  $\forall \epsilon_{1} > 0, \exists \delta_{1} > 0$ such that $\lvert \textrm{conv}f_{x_{0}}(u) - \textrm{conv}f_{x_{0}}(u_{0}) \rvert < \epsilon_{1}$ for all $u \in U$ where $\lVert u - u_{0} \rVert < \delta_{1}$.
  Additionally,
  since $f$ is continuous, $\forall \epsilon_{2} > 0, \exists \delta_{2} > 0$ such that $\lvert f_{x}(u) - f_{x_{0}}(u) \rvert < \epsilon_{2}$ for all $x \in X$ where $\lVert x - x_{0} \rVert < \delta_{2}$.
  Then,
  \begin{equation}
    \begin{split}
      \lvert f_{x}(u) - f_{x_{0}}(u) \rvert < \epsilon_{2}
      \Rightarrow
      f_{x}(u)
    &> f_{x_{0}}(u) - \epsilon_{2}
    \\
    &\geq \textrm{conv}f_{x_{0}}(u) - \epsilon_{2}
    \\
    &=: \textrm{conv}f_{x_{0}}^{\epsilon_{2}}(u),
    \end{split}
  \end{equation}
  due to the continuity of $f$.
  Since $\textrm{conv}f_{x_{0}}^{\epsilon_{2}}$ is convex
  and $\textrm{conv}f_{x_{0}}^{\epsilon_{2}} (u) \leq f_{x} (u), \forall u \in U$,
  $\textrm{conv}f_{x}(u) \geq \textrm{conv}f_{x_{0}}^{\epsilon_{2}}(u) = \textrm{conv}f_{x_{0}}(u) - \epsilon_{2}$.
  By symmetry, $\textrm{conv}f_{x_{0}}(u) \geq \textrm{conv}f_{x}(u) - \epsilon_{2} \Rightarrow \lvert \textrm{conv}f_{x}(u) - \textrm{conv}f_{x_{0}}(u) \rvert \leq \epsilon_{2} $.
  Therefore, $\forall \epsilon > 0$, let $\epsilon_{1} = \epsilon_{2} = \epsilon / 2$ and $\delta = \min\{ \delta_{1}, \delta_{2}\} > 0$,
  \begin{equation}
    \begin{split}
      \lvert \textrm{pconv}&f(x, u) - \textrm{pconv}f(x_{0}, u_{0}) \rvert
      \\
      =&\lvert \textrm{conv}f_{x}(u) - \textrm{conv}f_{x_{0}}(u_{0}) \rvert
      \\
      \leq& \lvert \textrm{conv}f_{x}(u) - \textrm{conv}f_{x_{0}}(u) \rvert
      \\
          &+ \lvert \textrm{conv}f_{x_{0}}(u) - \textrm{conv}f_{x_{0}}(u_{0}) \rvert
          \\
      <& \epsilon / 2 + \epsilon / 2
      = \epsilon,
    \end{split}
  \end{equation}
  for all $(x, u) \in X \times U$ where $\lVert (x, u) - (x_{0}, u_{0}) \rVert < \delta$,
  which concludes the proof.
\end{proof}

The following lemmas describe the characteristics of the minimizers and minimum values of the PGCM.
\begin{lemma}
  \label{lemma:pgcm_minimizer_minimum}
  Given continuous function $f: X \times U \to \mathbb{R}$,
  for any $x \in X$,
  the following holds true:
  \begin{itemize}
    \item $\min_{u \in U} f(x, u) = \min_{u \in U} \textup{pconv}f(x, u)$,
    \item $\argmin_{u \in U}f(x, u) \subset \argmin_{u \in U}\textup{pconv}f(x, u) $.
  \end{itemize}
\end{lemma}
\begin{proof}
  Let us define an auxiliary function $f_{a} : X \times U \to \mathbb{R}$ such that $f_{a}(x, u) := \min_{u' \in U} f(x, u')$.
  That is, $f_{a}(x, \cdot)$ is constant for any $x \in X$.
  Let us note that $f_{a}$ is well-defined by extreme value theorem~\cite[Theorem 4.16]{rudinPrinciplesMathematicalAnalysis1976}
  and also that $f_{a}$ is continuous by Berge's maximum theorem~\cite[Chapter E.3]{okRealAnalysisEconomic2007}.
  It is straightforward to show that $f_{a}$ is parameterized convex and $f_{a}(x, u) \leq f(x, u)$, $\forall (x, u) \in X \times U$,
  i.e., $f_{a}$ is a parameterized convex minorant of the given function $f$.
  By definition,
  $\text{pconv}f(x, u) \leq f(x, u)$, $\forall (x, u) \in X \times U$.
  This implies that
  $\text{pconv}f(x, u^{\star}) \leq f(x, u^{\star}) = \min_{u \in U} f(x, u)$,
  $\forall u^{\star} \in \argmin_{u \in U} f(x, u)$.

  If there exists $u_{1} \in U$
  such that $\text{pconv}f(x, u_{1}) < \min_{u \in U}f(x, u)$,
  then this contradicts the definition of PGCM because $\text{pconv}f(x, u_{1}) < \min_{u \in U}f(x, u) = f_{a}(x, u_{1})$
  for a parameterized convex minorant $f_{a}$.
  Therefore, $\text{pconv}f(x, u) \geq \min_{u \in U}f(x, u)$, $\forall (x, u) \in X \times U$.
  Hence,
  $\text{pconv}f(x, u^{\star}) = \min_{u \in U}f (x, u)$, $\forall x \in X$, $u^{\star} \in \argmin_{u \in U} f(x, u)$.
  This implies that for given $x \in X$, any $u^{\star} \in \argmin_{u \in U}f(x, u) $ attains the minimum of $\text{pconv}f$
  with the minimum value of $\min_{u \in U} \text{pconv}f(x, u) = \min_{u \in U} f(x, u) $,
  which concludes the proof.
\end{proof}

Then, the proof of \autoref{thm:universal_approximation_theorem} can be shown as follows.
\begin{proof}
  Fix $\epsilon > 0$.
  By \lref{lemma:pconv_is_parameterized_convex} and \lref{lemma:pconv_is_continuous},
  $\text{pconv}f$ is parameterized convex continuous.
  From Berge's maximum theorem,
  $\text{pconv}f^{\star}: X \to \mathbb{R}$
  is continuous where $\text{pconv}f^{\star}(x) := \min_{u \in U} \text{pconv}f(x, u)$, $\forall x \in X$~\cite[Chapter E.3]{okRealAnalysisEconomic2007}.
  Functions $f$, $\text{pconv}f$, and $\text{pconv}f^{\star}$ are continuous
  on compact sets $X \times U$, $X \times U$, and $X$, respectively~\cite{munkresTopology2014}.
  A fortiori,
  $f$, $\text{pconv}f$, and $\text{pconv}f^{\star}$ are uniformly continuous.
  Therefore, given $(x, u) \in X \times U$,
  $\forall \epsilon_{1} > 0$, $\exists \delta > 0$ such that
  \begin{equation}
    \label{eq:continuity}
    \begin{split}
      \lvert f(x_{1}, u_{1}) - f(x, u) \rvert &< \epsilon_{1},
      \\
      \lvert \text{pconv}f(x_{1}, u_{1}) - \text{pconv}f(x, u) \rvert &< \epsilon_{1},
      \\
      \lvert \text{pconv}f^{\star}(x_{1}) - \text{pconv}f^{\star}(x) \rvert &< \epsilon_{1},
    \end{split}
  \end{equation}
  for all $(x_{1}, u_{1}) \in X \times U$ where $\lVert (x_{1}, u_{1}) - (x, u) \rVert < \delta$.

  Let us define
  \begin{equation}
    \label{eq:overline_pconv}
    \begin{split}
      \overline{\text{pconv}f}(x, u)
      :=& \max (\text{pconv}f(x, u), \text{pconv}f(x, h_{\delta}(x)))
      \\
      &+ \gamma \lVert u - h_{\delta}(x) \rVert,
    \end{split}
  \end{equation}
  where $\gamma > 0$ is a positive constant,
  and $h_{\delta}$ is a continuous approximate selection in \aref{assumption:regularity_condition}.
  It is straightforward to show that $\overline{\text{pconv}f}$ is parameterized convex continuous.
  The given continuous function $f$ can be factorized as follows,
  \begin{equation}
    \label{eq:factorization}
    \begin{split}
      f(x, u)
    &= \overline{\text{pconv}f}(x, u)
    + \left(f(x, u) - \overline{\text{pconv}f}(x, u)\right)
    \\
    &=: \overline{\text{pconv}f}(x, u) + \overline{\Delta} f(x, u).
    \end{split}
  \end{equation}
  The proof shows that the factorized function is approximated by $f^{\text{PLSE}}$ and $f^{\text{NN}}$ through several steps.

  \textbf{(Step 1)}
  From \aref{assumption:regularity_condition}, \lref{lemma:pgcm_minimizer_minimum}, and \eqref{eq:continuity},
  given $x \in X$,
  there exists $x_{1} \in X$
  and $u_{1}^{\star} \in \argmin_{u \in U} f(x_{1}, u) \subset \argmin_{u \in U} \text{pconv}f(x_{1}, u) $
  such that $\lVert (x_{1}, u_{1}^{\star}) - (x, h_{\delta}(x)) \rVert < \delta$.
  Then,
  \begin{equation}
    \label{eq:step1_aux}
    \begin{split}
      0 \leq& \text{pconv}f(x, h_{\delta}(x)) - \text{pconv}f^{\star}(x)
      \\
        =& \lvert \text{pconv}f(x, h_{\delta}(x)) - \text{pconv}f^{\star}(x) \rvert
        \\
        \leq& \lvert \text{pconv}f(x, h_{\delta}(x)) - \text{pconv}f(x_{1}, u_{1}^{\star}) \rvert
        \\
        &+ \lvert \text{pconv}f(x_{1}, u_{1}^{\star}) - \text{pconv}f^{\star}(x) \rvert
        \\
        <& \epsilon_{1}
        + \lvert \text{pconv}f(x_{1}, u_{1}^{\star}) - \text{pconv}f^{\star}(x) \rvert
        \\
        =& \epsilon_{1}
        + \lvert \text{pconv}f^{\star}(x_{1}) - \text{pconv}f^{\star}(x) \rvert
        \\
        <& \epsilon_{1} + \epsilon_{1} = 2 \epsilon_{1}.
    \end{split}
  \end{equation}
  Therefore, using \eqref{eq:overline_pconv} and \eqref{eq:step1_aux},
  for any $(x, u) \in X \times U$,
  \begin{equation}
    \label{eq:aux_overline_approx}
    \begin{split}
      \lvert \overline{\text{pconv}f}&(x, u) - \text{pconv}f(x, u) \rvert
      \\
      =& \max(0, \text{pconv}f(x, h_{\delta}(x)) - \text{pconv}f(x, u)) 
      \\
       &+ \gamma \lVert u - h_{\delta}(x) \rVert
      \\
      \leq& \left(\text{pconv}f(x, h_{\delta}(x)) - \text{pconv}f^{\star}(x)\right)
      \\
          &+ \gamma \text{diam}(U)
          \\
      <& 2 \epsilon_{1} + \gamma \text{diam}(U),
    \end{split}
  \end{equation}
  where $\text{diam}(U)$ is the diameter of $U$ such that $\text{diam}(U) := \sup_{u, u' \in U} \lVert u - u' \rVert$.
  $U \subset \mathbb{R}^{m}$ is assumed to be compact,
  implying $U$ is bounded~\cite[Theorem 27.3]{munkresTopology2014}.
  Therefore, $\text{diam}(U) < \infty$.

  \textbf{(Step 2)}
  Since $\overline{\text{pconv}f}$ is parameterized convex continuous,
  for all $\epsilon_{2} > 0$,
  there exists a PCM $f^{\text{PCM}}$
  such that
  \begin{equation}
    \label{eq:plse_approx}
    \lVert f^{\text{PCM}} - \overline{\text{pconv}f} \rVert_{\infty} < \epsilon_{2},
  \end{equation}
  as the PCM $f^{\text{PCM}}$ is a shape-preserving universal approximator for parameterized convex continuous functions.
  By \eqref{eq:plse_approx} and the definition of $\overline{\text{pconv}f}$ in \eqref{eq:overline_pconv},
  \begin{equation}
    \text{pconv}f(x, h_{\delta}(x))
    = \overline{\text{pconv}f}(x, h_{\delta}(x))
    \geq f^{\text{PCM}}(x, h_{\delta}(x)) - \epsilon_{2},
  \end{equation}
  and therefore, from \eqref{eq:plse_approx},
  for any $(x, u) \in X \times U$,
  \begin{equation}
    \begin{split}
    f^{\text{PCM}}&(x, u)
    \geq \overline{\text{pconv}f}(x, u) - \epsilon_{2}
    \\
    =& \max( \text{pconv}f(x, u), \text{pconv}f(x, h_{\delta}(x)))
    \\
     &+ \gamma \lVert u - h_{\delta}(x) \rVert - \epsilon_{2}
    \\
    \geq& \text{pconv}f(x, h_{\delta}(x)) + \gamma \lVert u - h_{\delta}(x) \rVert - \epsilon_{2}
    \\
    \geq& \left(f^{\text{PCM}}(x, h_{\delta}(x)) - \epsilon_{2}\right) + \gamma \lVert u - h_{\delta}(x) \rVert - \epsilon_{2}
    \\
    =& f^{\text{PCM}}(x, h_{\delta}(x)) + \gamma \lVert u - h_{\delta}(x) \rVert - 2\epsilon_{2},
    \end{split}
  \end{equation}
  which implies for all $x \in X$ and $\hat{u}^{\star} \in \argmin_{u \in U}f^{\text{PCM}}(x, u) $ that
  \begin{equation}
    \label{eq:step2}
    \begin{split}
      0
      &\geq f^{\text{PCM}}(x, \hat{u}^{\star}) - f^{\text{PCM}}(x, h_{\delta}(x))
      \\
      &\geq \gamma \lVert \hat{u}^{\star} - h_{\delta}(x) \rVert - 2 \epsilon_{2}
      \\
      &\Rightarrow \lVert \hat{u}^{\star} - h_{\delta}(x) \rVert \leq 2 \epsilon_{2} / \gamma.
    \end{split}
  \end{equation}

  \textbf{(Step 3)}
  From the definition of $\overline{\Delta}f$ in \eqref{eq:factorization},
  $\overline{\Delta}f$ is continuous on $X \times U$ compact.
  A fortiori, $\overline{\Delta}f$ is uniformly continuous.
  Hence, given $(x, u) \in X \times U$,
  $\forall \epsilon_{3} > 0$, $\exists \delta_{3} > 0$ such that
  \begin{equation}
    \label{eq:continuity_of_overline_Delta_f}
    \lvert \overline{\Delta}f (x_{3}, u_{3}) - \overline{\Delta}f (x, u) \rvert < \epsilon_{3},
  \end{equation}
  for all $ (x_{3}, u_{3}) \in X \times U $ where $\lVert (x_{3}, u_{3}) - (x, u) \rVert < \delta_{3}$.
  From \eqref{eq:step2} and \eqref{eq:continuity_of_overline_Delta_f},
  setting $\epsilon_{2} = \min (\delta_{3} \gamma/2, \epsilon / 12)$ implies
  \begin{equation}
    \label{eq:step3}
    \lvert \overline{\Delta}f(x, \hat{u}^{\star}) - \overline{\Delta}f(x, h_{\delta}(x)) \rvert < \epsilon_{3},
  \end{equation}
  for all $x \in X$ and $\hat{u}^{\star} \in \argmin_{u \in U}f^{\text{PCM}}(x, u)$.

  \textbf{(Step 4)}
  Since $\overline{\Delta}f$ is continuous,
  for all $\epsilon_{4} > 0$,
  there exists a shape-preserving universal approximator for continuous functions, $f^{\text{NN}}$, such that
  \begin{equation}
    \label{eq:nn_approx}
    \lVert f^{\text{NN}} - \overline{\Delta}f \rVert_{\infty} < \epsilon_{4}.
  \end{equation}

  From \aref{assumption:regularity_condition}, \eqref{eq:continuity}, \eqref{eq:overline_pconv}, and \lref{lemma:pgcm_minimizer_minimum},
  for all $x \in X$,
  there exists $x_{1} \in X$ and $u^{\star}_{1} \in \argmin_{u \in U} f(x_{1}, u)$
  where $\lVert (x_{1}, u_{1}^{\star}) - (x, h_{\delta}(x)) \rVert < \delta$ such that
  \begin{equation}
    \label{eq:aux_overline_delta_f}
    \begin{split}
      \lvert \overline{\Delta}  f&(x, h_{\delta}(x)) \rvert
      = \lvert f(x, h_{\delta}(x)) - \overline{\text{pconv}f}(x, h_{\delta}(x)) \rvert
      \\
      \leq& \lvert f(x, h_{\delta}(x)) - f(x_{1}, u^{\star}_{1}) \rvert
      \\
      &+ \lvert f(x_{1}, u^{\star}_{1}) - \overline{\text{pconv}f}(x, h_{\delta}(x)) \rvert
      \\
      <& \epsilon_{1} + \lvert f(x_{1}, u^{\star}_{1}) - \overline{\text{pconv}f}(x, h_{\delta}(x)) \rvert
      \\
      =& \epsilon_{1} + \lvert f(x_{1}, u^{\star}_{1}) - \text{pconv}f(x, h_{\delta}(x)) \rvert
      \\
      =& \epsilon_{1} + \lvert \text{pconv}f(x_{1}, u^{\star}_{1}) - \text{pconv}f(x, h_{\delta}(x)) \rvert
      \\
      <& \epsilon_{1} + \epsilon_{1} = 2 \epsilon_{1}.
    \end{split}
  \end{equation}

  Therefore,
  using \eqref{eq:step3}, \eqref{eq:nn_approx} and \eqref{eq:aux_overline_delta_f}
  implies that for all $(x, u) \in X \times U$ and $\hat{u}^{\star} \in \argmin_{u \in U} f^{\text{PCM}}(x, u)$,
  \begin{equation}
    \label{eq:aux_before_relu}
    \begin{split}
    \lvert f^{\text{NN}}&(x, u) - f^{\text{NN}}  (x, \hat{u}^{\star}) - \overline{\Delta}f(x, u) \rvert
    \\
    \leq& \lvert f^{\text{NN}}(x, u) - \overline{\Delta}f(x, u) \rvert
    + \lvert f^{\text{NN}}(x, \hat{u}^{\star}) \rvert
    \\
    <& \epsilon_{4} + \lvert f^{\text{NN}}(x, \hat{u}^{\star}) \rvert
    \\
    \leq& \epsilon_{4}
    + \lvert f^{\text{NN}}(x, \hat{u}^{\star}) - \overline{\Delta}f(x, \hat{u}^{\star}) \rvert
    + \lvert \overline{\Delta}f(x, \hat{u}^{\star}) \rvert
    \\
    <& 2\epsilon_{4} + \lvert \overline{\Delta}f(x, \hat{u}^{\star}) \rvert
    \\
    \leq& 2\epsilon_{4}
    + \lvert \overline{\Delta}f(x, \hat{u}^{\star}) - \overline{\Delta}f(x, h_{\delta}(x))  \rvert
    \\
    &+ \lvert \overline{\Delta}f(x, h_{\delta}(x)) \rvert
    \\
    <& \epsilon_{3} + 2\epsilon_{4} + \lvert \overline{\Delta}f(x, h_{\delta}(x)) \rvert
    \\
    <& 2\epsilon_{1} + \epsilon_{3} + 2\epsilon_{4}.
    \end{split}
  \end{equation}
  
  \textbf{(Step 5)}
  Now, let us show that
  \begin{equation}
    \label{eq:aux_with_relu}
    \begin{split}
      A
    &:= \lvert \max(0, f^{\text{NN}}(x, u) - f^{\text{NN}}(x, \hat{u}^{\star}(x))) - \overline{\Delta}f(x, u) \rvert
    \\
    &< 6 \epsilon_{1} + \epsilon_{3} + 2 \epsilon_{4} + 2 \gamma \text{diam}(U),
    \end{split}
  \end{equation}
  where $\hat{u}^{\star}:X \to U$ is a single-valued minimizer function of the PCM $f^{\text{PCM}}$
  defined in \eqref{eq:pcm_method} such that $\hat{u}^{\star}(x) \in \argmin_{u \in U} f^{\text{PCM}}(x, u)$ (with abuse of notation).

  If $f^{\text{NN}}(x, u) - f^{\text{NN}}(x, \hat{u}^{\star}(x)) \geq 0$,
  then from \eqref{eq:aux_before_relu},
  \begin{equation}
    \begin{split}
      A
      &= \lvert f^{\text{NN}}(x, u) - f^{\text{NN}}(x, \hat{u}^{\star}(x)) - \overline{\Delta}f(x, u) \rvert
      \\
      &< 2\epsilon_{1} + \epsilon_{3} + 2\epsilon_{4}
      \\
      &< 6 \epsilon_{1} + \epsilon_{3} + 2 \epsilon_{4} + 2 \gamma \text{diam}(U).
    \end{split}
  \end{equation}

  If $f^{\text{NN}}(x, u) - f^{\text{NN}}(x, \hat{u}^{\star}(x)) \leq 0$,
  it can be shown from \eqref{eq:aux_overline_approx} that
  \begin{equation}
    \begin{split}
      \lVert \overline{\Delta}f - \Delta f \rVert_{\infty}
      &= \lVert \overline{\text{pconv}f} - \text{pconv}f \rVert_{\infty}
      \\
      &< 2 \epsilon_{1} + \gamma \text{diam}(U) =: \epsilon_{1, \gamma},
    \end{split}
  \end{equation}
  where $\Delta f(x, u) := f(x, u) - \text{pconv}f(x, u) \geq 0$, $\forall (x, u) \in X \times U$,
  which implies
  $\overline{\Delta}f(x, u) + \epsilon_{1, \gamma}
  \geq \Delta f(x, u)
  \geq 0,$
  $\forall (x, u) \in X \times U$.
  Therefore,
  using \eqref{eq:aux_before_relu} implies
  \begin{equation}
    \label{eq:step5}
    \begin{split}
      A
      &= \lvert \max(0, f^{\text{NN}}(x, u) - f^{\text{NN}}(x, \hat{u}^{\star}(x))) - \overline{\Delta}f(x, u) \rvert
      \\
      &= \lvert \overline{\Delta}f(x, u) \rvert
      = \lvert \overline{\Delta}f(x, u) + \epsilon_{1, \gamma} - \epsilon_{1, \gamma} \rvert
      \\
      &\leq \left(\overline{\Delta}f(x, u) + \epsilon_{1, \gamma}\right) + \epsilon_{1, \gamma}
      \\
      &= \overline{\Delta}f(x, u)
      + 2\epsilon_{1, \gamma}
      \\
      &\leq \overline{\Delta}f(x, u) - (f^{\text{NN}}(x, u) - f^{\text{NN}}(x, \hat{u}^{\star}(x)))
      + 2 \epsilon_{1, \gamma}
      \\
      &= -\left(
        f^{\text{NN}}(x, u)
        - f^{\text{NN}}(x, \hat{u}^{\star}(x)) - \overline{\Delta}f(x, u)
      \right) 
      + 2 \epsilon_{1, \gamma}
      \\
      &\leq \lvert f^{\text{NN}}(x, u) - f^{\text{NN}}(x, \hat{u}^{\star}(x)) - \overline{\Delta}f(x, u) \rvert
      + 2 \epsilon_{1, \gamma}
      \\
      &< 2 \epsilon_{1} + \epsilon_{3} + 2 \epsilon_{4} + 2 \epsilon_{1, \gamma}
      \\
      &= 6 \epsilon_{1} + \epsilon_{3} + 2 \epsilon_{4} + 2 \gamma \text{diam}(U).
    \end{split}
  \end{equation}

  \textbf{(Step 6)}
  For all $(x, u) \in X \times U$,
  Using \eqref{eq:pcm_method}, \eqref{eq:factorization}, \eqref{eq:plse_approx}, \eqref{eq:aux_with_relu}, and \eqref{eq:step5} implies
  \begin{equation}
    \begin{split}
      \lvert \hat{f}&(x, u) - f(x, u) \rvert
      \\
      =& \vert \left(f^{\text{PCM}}(x, u) + \max(0, f^{\text{NN}}(x, u) - f^{\text{NN}}(x, \hat{u}^{\star}(x)))\right)
      \\
      &- \left(\overline{\text{pconv}f}(x, u) + \overline{\Delta} f(x, u)\right) \vert
      \\
      \leq& \lvert f^{\text{PCM}}(x, u) - \overline{\text{pconv}f}(x, u) \rvert
      \\
      &+ \lvert \max(0, f^{\text{NN}}(x, u) - f^{\text{NN}}(x, \hat{u}^{\star}(x))) - \overline{\Delta}f(x, u) \rvert
      \\
      =& \lvert f^{\text{PCM}}(x, u) - \overline{\text{pconv}f}(x, u) \rvert
      + A
      \\
      <& \epsilon_{2} + A
      \\
      <& \epsilon_{2} + \left( 6 \epsilon_{1} + \epsilon_{3} + 2 \epsilon_{4} + 2 \gamma \text{diam}(U)\right)
      \\
      =&  6 \epsilon_{1} + \epsilon_{2} + \epsilon_{3} + 2 \epsilon_{4} + 2 \gamma \text{diam}(U)
      \\
      \leq& \left(\frac{6}{12} + \frac{1}{12} + \frac{1}{12} + \frac{2}{12} + \frac{2}{12}\right) \epsilon
      = \epsilon,
    \end{split}
  \end{equation}
  by substituting $\epsilon_{1} = \epsilon_{3} = \epsilon_{4} = \epsilon / 12$ and $\gamma = \frac{\epsilon}{12 \text{diam}(U)}$,
  which concludes the proof.
\end{proof}

%% file: appendices/appendix_B.tex
\section{Proof of \autoref{thm:uat_of_plse_plus}}
\setcounter{equation}{0}  %
\renewcommand{\theequation}{\thesection.\arabic{equation}}  %
\label{sec:proof_of_uat_plseplus}
\begin{proof}
  The following proof is a modified version of the proof of \cite[Theorem 3]{kimParameterizedConvexUniversal2022b},
  and some notations are borrowed from the proof without clarification.
  By Berge's maximum theorem,
  $f^{\star}: X \to \mathbb{R}$ is continuous on $X$ where $f^{\star}(x) := \min_{u \in U} f(x, u)$~\cite[Chapter E.3]{okRealAnalysisEconomic2007}.
  A fortiori, $f^{\star}$ is uniformly continuous.
  From the uniform continuity of $f^{\star}$ and the fact that $(x, u) \mapsto f^{\star}(x)$ is an underestimator of $f$,
  it is straightforward to show that $\{f^{\text{aux}}_{i}\}_{i \in \mathbb{N}}$ is equicontinuous
  where $f^{\text{aux}}_{i}(x, u) = f^{\star}(x)$ for $i=1$
  and $f^{\text{aux}}_{i}(x, u) = f(x, u_{i}) + \langle \hat{u}^{*}_{\epsilon, i}(x), u - u_{i} \rangle$ for $i \in \mathbb{N} \setminus \{1\} $.
  A fortiori,
  $\{(x, u) \mapsto \sup_{1 \leq j \leq i}f^{\text{aux}}_{j}(x, u)\}_{i \in \mathbb{N}}$ is equicontinuous.
  Additionally,
  $(x, u) \mapsto \sup_{1 \leq i \leq I} f^{\text{aux}}_{i}(x, u)$ converges pointwise to $f$ on $X \times \tilde{U}$.
  Following the rest of the proof of \cite[Theorem 3]{kimParameterizedConvexUniversal2022b}
  concludes the proof.
\end{proof}

%% file: main.bbl
\begin{thebibliography}{10}
\providecommand{\url}[1]{#1}
\csname url@samestyle\endcsname
\providecommand{\newblock}{\relax}
\providecommand{\bibinfo}[2]{#2}
\providecommand{\BIBentrySTDinterwordspacing}{\spaceskip=0pt\relax}
\providecommand{\BIBentryALTinterwordstretchfactor}{4}
\providecommand{\BIBentryALTinterwordspacing}{\spaceskip=\fontdimen2\font plus
\BIBentryALTinterwordstretchfactor\fontdimen3\font minus
  \fontdimen4\font\relax}
\providecommand{\BIBforeignlanguage}[2]{{%
\expandafter\ifx\csname l@#1\endcsname\relax
\typeout{** WARNING: IEEEtran.bst: No hyphenation pattern has been}%
\typeout{** loaded for the language `#1'. Using the pattern for}%
\typeout{** the default language instead.}%
\else
\language=\csname l@#1\endcsname
\fi
#2}}
\providecommand{\BIBdecl}{\relax}
\BIBdecl

\bibitem{amosTutorialAmortizedOptimization2023}
B.~Amos, ``\BIBforeignlanguage{en}{Tutorial on {Amortized} {Optimization}},''
  Apr. 2023, arXiv:2202.00665 [cs, math].

\bibitem{kimParameterizedConvexUniversal2022b}
J.~Kim and Y.~Kim, ``\BIBforeignlanguage{en}{Parameterized {Convex} {Universal}
  {Approximators} for {Decision}-{Making} {Problems}},''
  \emph{\BIBforeignlanguage{en}{IEEE Transactions on Neural Networks and
  Learning Systems}}, 2022.

\bibitem{boydConvexOptimization2004}
S.~P. Boyd and L.~Vandenberghe, \emph{\BIBforeignlanguage{en}{Convex
  {Optimization}}}.\hskip 1em plus 0.5em minus 0.4em\relax Cambridge, UK:
  Cambridge University Press, 2004.

\bibitem{malyutaConvexOptimizationTrajectory2022}
D.~Malyuta, T.~P. Reynolds, M.~Szmuk, T.~Lew, R.~Bonalli, M.~Pavone, and
  B.~Açıkmeşe, ``\BIBforeignlanguage{en}{Convex {Optimization} for
  {Trajectory} {Generation}: {A} {Tutorial} on {Generating} {Dynamically}
  {Feasible} {Trajectories} {Reliably} and {Efficiently}},''
  \emph{\BIBforeignlanguage{en}{IEEE Control Systems}}, vol.~42, no.~5, pp.
  40--113, 2022.

\bibitem{liuSurveyConvexOptimization2017}
X.~Liu, P.~Lu, and B.~Pan, ``\BIBforeignlanguage{en}{Survey of {Convex}
  {Optimization} for {Aerospace} {Applications}},''
  \emph{\BIBforeignlanguage{en}{Astrodynamics}}, vol.~1, no.~1, pp. 23--40,
  2017.

\bibitem{calafioreLogSumExpNeuralNetworks2020}
G.~C. Calafiore, S.~Gaubert, and C.~Possieri, ``Log-{Sum}-{Exp} {Neural}
  {Networks} and {Posynomial} {Models} for {Convex} and {Log}-{Log}-{Convex}
  {Data},'' \emph{IEEE Transactions on Neural Networks and Learning Systems},
  vol.~31, no.~3, pp. 827--838, 2020.

\bibitem{calafioreEfficientModelFreeQFactor2020}
G.~C. Calafiore and C.~Possieri, ``\BIBforeignlanguage{en}{Efficient
  {Model}-{Free} {Q}-{Factor} {Approximation} in {Value} {Space} via
  {Log}-{Sum}-{Exp} {Neural} {Networks}},'' in
  \emph{\BIBforeignlanguage{en}{2020 {European} {Control} {Conference}
  ({ECC})}}.\hskip 1em plus 0.5em minus 0.4em\relax Saint Petersburg, Russia:
  IEEE, May 2020.

\bibitem{kimVTOLAircraftOptimal2023}
J.~Kim, H.~Lee, and Y.~Kim, ``\BIBforeignlanguage{en}{{VTOL} {Aircraft}
  {Optimal} {Gain} {Prediction} via {Parameterized} {Log}-{Sum}-{Exp}
  {Networks}},'' in \emph{\BIBforeignlanguage{en}{European {Control}
  {Conference} ({ECC})}}, Bucharest, Romania, Jun. 2023.

\bibitem{calafioreUniversalApproximationResult2020}
G.~C. Calafiore, S.~Gaubert, and C.~Possieri, ``A {Universal} {Approximation}
  {Result} for {Difference} of {Log}-{Sum}-{Exp} {Neural} {Networks},''
  \emph{IEEE Transactions on Neural Networks and Learning Systems}, vol.~31,
  no.~12, pp. 5603--5612, 2020.

\bibitem{lethiDCProgrammingDCA2018}
H.~A. Le~Thi and T.~Pham Dinh, ``\BIBforeignlanguage{en}{{DC} {Programming}
  and {DCA}: {Thirty} {Years} of {Developments}},''
  \emph{\BIBforeignlanguage{en}{Mathematical Programming}}, vol. 169, no.~1,
  pp. 5--68, May 2018.

\bibitem{abramsonConvexMinorantsRandom2011}
J.~Abramson, J.~Pitman, N.~Ross, and G.~Uribe~Bravo,
  ``\BIBforeignlanguage{en}{Convex {Minorants} of {Random} {Walks} and {Lévy}
  {Processes}},'' \emph{\BIBforeignlanguage{en}{Electronic Communications in
  Probability}}, vol.~16, Jan. 2011.

\bibitem{aubinSetValuedAnalysis2009}
J.-P. Aubin and H.~Frankowska, \emph{\BIBforeignlanguage{en}{Set-{Valued}
  {Analysis}}}.\hskip 1em plus 0.5em minus 0.4em\relax Boston, MA: Birkhäuser
  Boston, 2009.

\bibitem{suttonReinforcementLearningIntroduction2018}
R.~S. Sutton and A.~G. Barto, \emph{\BIBforeignlanguage{en}{Reinforcement
  {Learning}: {An} {Introduction}}}.\hskip 1em plus 0.5em minus 0.4em\relax
  Cambridge, MA: MIT press, 2018.

\bibitem{liberzonCalculusVariationsOptimal2012}
D.~Liberzon, \emph{\BIBforeignlanguage{en}{Calculus of {Variations} and
  {Optimal} {Control} {Theory}: {A} {Concise} {Introduction}}}.\hskip 1em plus
  0.5em minus 0.4em\relax Princeton, NJ: Princeton University Press, 2012.

\bibitem{camachoModelPredictiveControl2007}
E.~F. Camacho and C.~Bordons, \emph{Model {Predictive} {Control}}, ser.
  Advanced {Textbooks} in {Control} and {Signal} {Processing}, M.~J. Grimble
  and M.~A. Johnson, Eds.\hskip 1em plus 0.5em minus 0.4em\relax London, UK:
  Springer, 2007.

\bibitem{pinkusApproximationTheoryMLP1999}
A.~Pinkus, ``\BIBforeignlanguage{en}{Approximation {Theory} of the {MLP}
  {Model} in {Neural} {Networks}},'' \emph{\BIBforeignlanguage{en}{Acta
  Numerica}}, vol.~8, pp. 143--195, 1999.

\bibitem{hornikMultilayerFeedforwardNetworks1989}
K.~Hornik, M.~Stinchcombe, and H.~White, ``\BIBforeignlanguage{en}{Multilayer
  {Feedforward} {Networks} are {Universal} {Approximators}},''
  \emph{\BIBforeignlanguage{en}{Neural Networks}}, vol.~2, no.~5, pp. 359--366,
  1989.

\bibitem{domahidiECOSSOCPSolver2013}
A.~Domahidi, E.~Chu, and S.~Boyd, ``\BIBforeignlanguage{en}{{ECOS}: {An} {SOCP}
  {Solver} for {Embedded} {Systems}},'' in
  \emph{\BIBforeignlanguage{en}{European {Control} {Conference} ({ECC})}},
  Zurich, Switzerland, Jul. 2013.

\bibitem{kingmaAdamMethodStochastic2017}
D.~P. Kingma and J.~Ba, ``Adam: {A} {Method} for {Stochastic} {Optimization},''
  Jan. 2017, arXiv: 1412.6980 [cs].

\bibitem{agrawalDifferentiableConvexOptimization2019}
A.~Agrawal, B.~Amos, S.~Barratt, and S.~Boyd,
  ``\BIBforeignlanguage{en}{Differentiable {Convex} {Optimization} {Layers}},''
  in \emph{\BIBforeignlanguage{en}{33rd {Conference} on {Neural} {Information}
  {Processing} {Systems} ({NeurIPS} 2019)}}, Vancouver, Canada, Dec. 2019.

\bibitem{blondelEfficientModularImplicit2022}
M.~Blondel, Q.~Berthet, M.~Cuturi, R.~Frostig, S.~Hoyer, F.~Llinares-López,
  F.~Pedregosa, and J.-P. Vert, ``\BIBforeignlanguage{en}{Efficient and
  {Modular} {Implicit} {Differentiation}},'' Oct. 2022, arXiv:2105.15183 [cs,
  math, stat].

\bibitem{nielsenMonteCarloInformation2018}
F.~Nielsen and G.~Hadjeres, ``\BIBforeignlanguage{en}{Monte {Carlo}
  {Information} {Geometry}: {The} {Dually} {Flat} {Case}},''
  \emph{\BIBforeignlanguage{en}{arXiv:1803.07225 [cs, stat]}}, Mar. 2018.

\bibitem{kimOfflineDifferentiableQlearning2022}
J.~Kim, H.~Lee, Y.~Lee, and Y.~Kim, ``Offline {Differentiable} {Q}-learning for
  {Aircraft} {Control} {Design},'' in \emph{2022 {Asia}-{Pacific}
  {International} {Symposium} on {Aerospace} {Technology}}, Niigata, Japan,
  Oct. 2022.

\bibitem{kmogensenOptimMathematicalOptimization2018}
P.~K~Mogensen and A.~N~Riseth, ``\BIBforeignlanguage{en}{Optim: {A}
  {Mathematical} {Optimization} {Package} for {Julia}},''
  \emph{\BIBforeignlanguage{en}{Journal of Open Source Software}}, vol.~3,
  no.~24, p. 615, 2018.

\bibitem{tarnFuzzyControlWing1993}
J.~H. Tarn and F.~Y. Hsu, ``\BIBforeignlanguage{en}{Fuzzy {Control} of {Wing}
  {Rock} for {Slender} {Delta} wings},'' in \emph{\BIBforeignlanguage{en}{1993
  {American} {Control} {Conference}}}, San Francisco, CA, Jun. 1993.

\bibitem{rockafellarConvexAnalysis1970}
R.~T. Rockafellar, \emph{\BIBforeignlanguage{en}{Convex {Analysis}}},
  2nd~ed.\hskip 1em plus 0.5em minus 0.4em\relax Princeton, NJ: Princeton
  University Press, 1970.

\bibitem{rudinPrinciplesMathematicalAnalysis1976}
W.~Rudin, \emph{\BIBforeignlanguage{en}{Principles of {Mathematical}
  {Analysis}}}, 3rd~ed., ser. International {Series} in {Pure} and {Applied}
  {Mathematics}.\hskip 1em plus 0.5em minus 0.4em\relax New York, NY:
  McGraw-Hill, 1976.

\bibitem{okRealAnalysisEconomic2007}
E.~A. Ok, \emph{\BIBforeignlanguage{en}{Real {Analysis} with {Economic}
  {Applications}}}.\hskip 1em plus 0.5em minus 0.4em\relax Princeton, NJ:
  Princeton University Press, 2007, vol.~10.

\bibitem{munkresTopology2014}
J.~R. Munkres, \emph{\BIBforeignlanguage{en}{Topology}}, 2nd~ed.\hskip 1em plus
  0.5em minus 0.4em\relax Harlow, Essex, UK: Pearson, 2014.

\end{thebibliography}
